\documentclass[sigconf]{acmart}
\AtBeginDocument{%
  \providecommand\BibTeX{{%
    \normalfont B\kern-0.5em{\scshape i\kern-0.25em b}\kern-0.8em\TeX}}}

\setcopyright{acmlicensed}
\copyrightyear{2018}
\acmYear{2018}
\acmDOI{XXXXXXX.XXXXXXX}
\acmConference[Conference acronym 'XX]{Make sure to enter the correct
  conference title from your rights confirmation emai}{June 03--05,
  2018}{Woodstock, NY}
\acmISBN{978-1-4503-XXXX-X/18/06}

\usepackage{booktabs} 
\usepackage{graphicx} 
\usepackage{tabularx}
\usepackage{diagbox}
\usepackage{multirow}
\usepackage{amsmath,amsbsy}
\usepackage{subcaption}
\usepackage{svg}
\usepackage{pdfpages}
\usepackage{algorithm}
\usepackage{algpseudocode}
\usepackage[bottom]{footmisc} 

\theoremstyle{acmplain} 
\newtheorem{theorem}{Theorem}

\theoremstyle{acmplain}
\newtheorem{definition}{Definition}

\renewcommand\footnotetextcopyrightpermission[1]{} 

\newcommand{\cmtt}[1]{{\fontfamily{cmtt}\selectfont #1}}
\begin{document}
\fancyhead{}
\title{Generation is better than Modification: Combating High Class Homophily Variance in Graph Anomaly Detection}

\ccsdesc[500]{Information systems~Data mining}

\author{%
    Rui Zhang\textsuperscript{\rm 1},
    Dawei Cheng\textsuperscript{\rm 1,\rm2,\rm$*$},
    Xin Liu\textsuperscript{\rm 1},
    Jie Yang\textsuperscript{\rm 1},
    Yi Ouyang\textsuperscript{\rm 3},
    Xian Wu\textsuperscript{\rm 3},
    Yefeng Zheng\textsuperscript{\rm 3}}
\affiliation{
 \institution{
\textsuperscript{\rm 1}Department of Computer Science and Technology, Tongji University, Shanghai, China\\
    \textsuperscript{\rm 2}Shanghai Artificial Intelligence Laboratory, Shanghai, China\\
    \textsuperscript{\rm 3}Jarvis Research Center, Tencent YouTu Lab, Shenzhen, China\\
 }
\country{}
}
\affiliation{
 \institution{
  \{2050271,dcheng,2051277,2153814\}@tongji.edu.cn, 
    \{yiouyang,kevinxwu,yefengzheng\}@tencent.com
    }
    \country{}
}

\renewcommand{\shortauthors}{Rui Zhang, Dawei Cheng, Xin Liu, Jie Yang, Yi Ouyang, Xian Wu and Yefeng Zheng}

\begin{abstract}
Graph-based anomaly detection is currently an important research topic in the field of graph neural networks (GNNs). We find that in graph anomaly detection, the homophily distribution differences between different classes are significantly greater than those in homophilic and heterophilic graphs. For the first time, we introduce a new metric called \textbf{Class Homophily Variance}, which quantitatively describes this phenomenon. To mitigate its impact, we propose a novel GNN model named Homophily Edge Generation Graph Neural Network (\cmtt{HedGe}). 
Previous works typically focused on pruning, selecting or connecting on original relationships, and we refer to these methods as modifications. Different from these works, our method emphasizes generating new relationships with low class homophily variance, using the original relationships as an auxiliary.
\cmtt{HedGe} samples homophily adjacency matrices from scratch using a self-attention mechanism, and leverages nodes that are relevant in the feature space but not directly connected in the original graph. Additionally, we modify the loss function to punish the generation of unnecessary heterophilic edges by the model. 
Extensive comparison experiments demonstrate that \cmtt{HedGe} achieved the best performance across multiple benchmark datasets, including anomaly detection and edgeless node classification. The proposed model also improves the robustness under the novel Heterophily Attack with increased class homophily variance on other graph classification tasks.
\footnotetext{Corresponding author.}
\end{abstract}

\maketitle


\keywords{Graph Anomaly Detection, Graph Neural Network, Homophily and Heterophily}




\section{Introduction}

\begin{figure*}[ht]
    \centering
    \begin{minipage}{\textwidth}
        \begin{subfigure}[b]{0.33\linewidth}
            \centering
            \includegraphics[width=\linewidth]{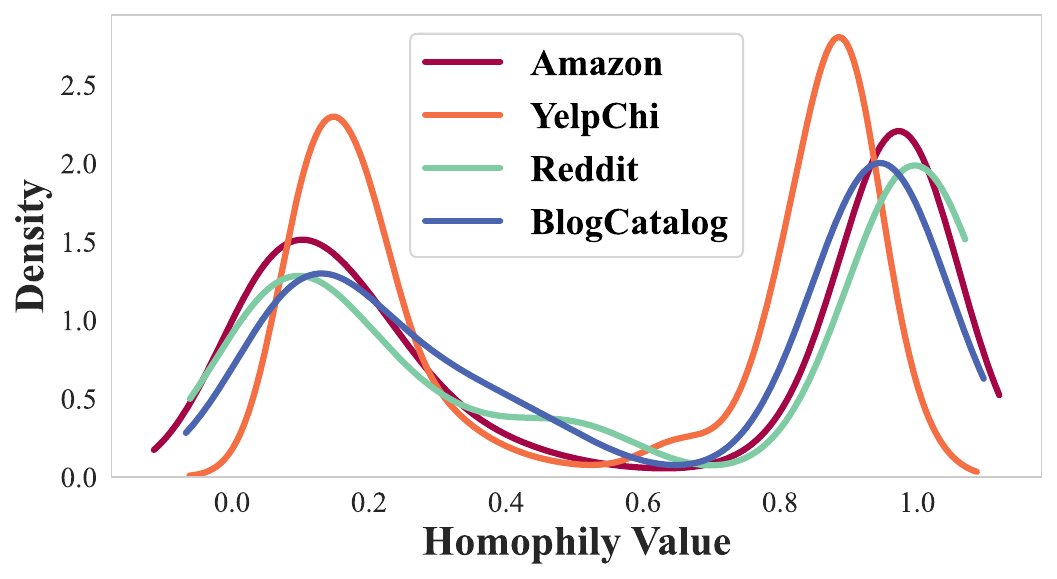}
            \caption{Anomaly detection datasets}
            \label{fig:image1}
        \end{subfigure}
        \hfill
        \begin{subfigure}[b]{0.33\linewidth}
            \centering
            \includegraphics[width=\linewidth]{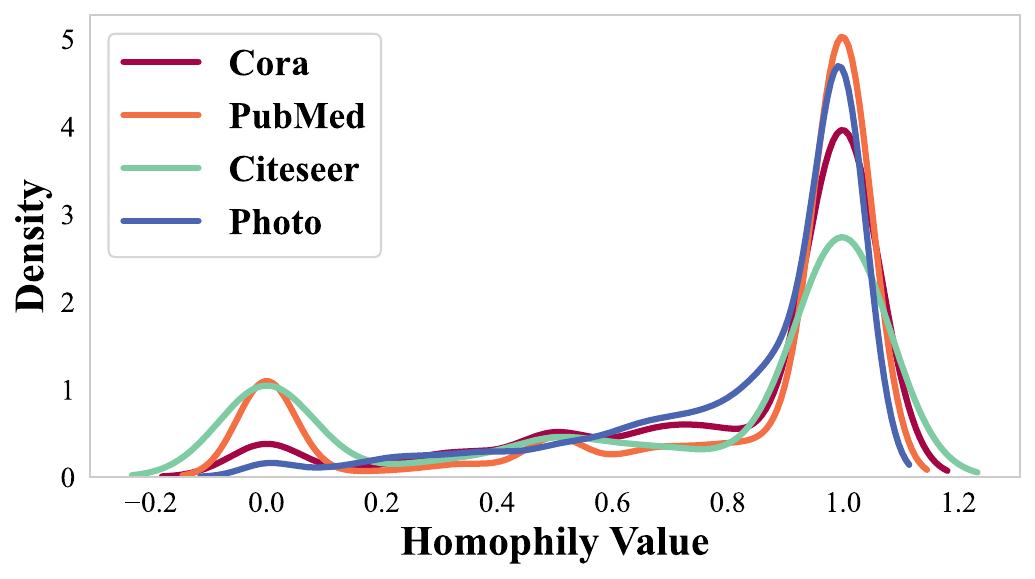}
            \caption{Homophilic graphs}
            \label{fig:image2}
        \end{subfigure}
        \hfill
        \begin{subfigure}[b]{0.33\linewidth}
            \centering
            \includegraphics[width=\linewidth]{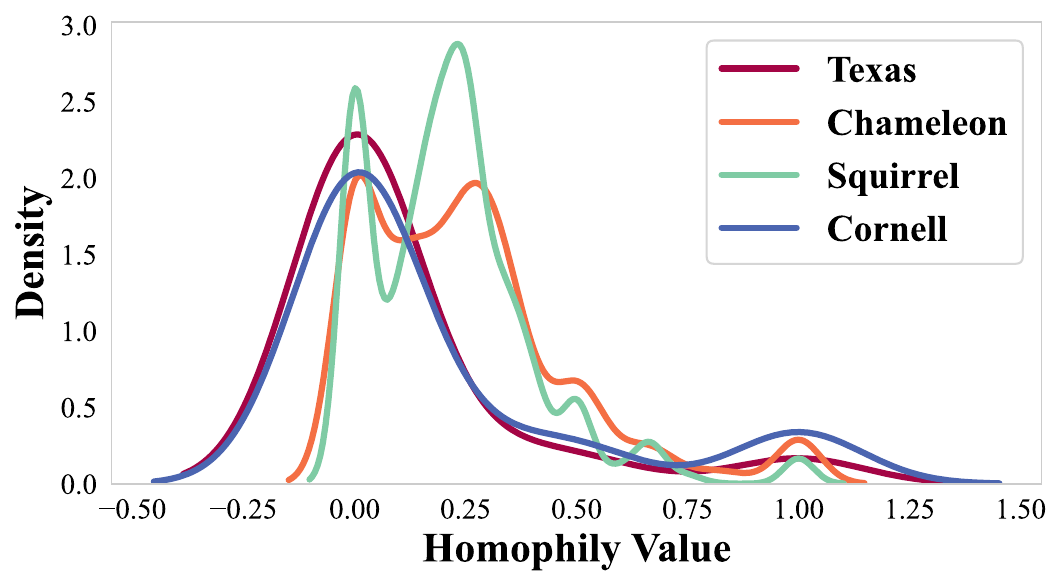}
            \caption{Heterophilic graphs}
            \label{fig:image3}
        \end{subfigure}
        \vspace{-12pt}
        \caption{Weighted homophily density distribution on different datasets. We use curves to fit the distribution for clarity.}
        \label{fig:fraud_homo}
    \end{minipage}
    \vspace{-8pt}
\end{figure*}

In graph anomaly detection (GAD), anomalous nodes refer to those in a network whose behavior or attributes significantly differ from most other nodes \cite{chandola2009anomaly}. GAD is important in fields like financial fraud detection \cite{cheng2023anti,deng2021graph}, cybersecurity \cite{ten2011anomaly}, social network analysis \cite{yang2019mining}, loan risk assessment \cite{cheng2020risk,niu2020iconviz} and industrial system monitoring \cite{chen2021interaction}, and has drawn great research interests. Graph Neural Networks (GNNs), with their effective handling and analysis of graph-structured data, are particularly suitable for GAD. Their abilities to aggregate information from neighborhoods help effectively detect anomaly nodes \cite{ma2021comprehensive}.

In the field of GAD, a significant amount of research has achieved notable results. 
Many studies selectively aggregate neighbor features and utilize supervised learning or feature similarity to differentiate between various neighbor pairs \cite{liu2021pick,dou2020enhancing}. Additionally, models based on spectral architecture, which leverage the characteristics of low-pass and high-pass filters, have effectively handled different structures and features, providing a new perspective for GAD \cite{tang2022rethinking, chai2022can}. Moreover, a series of studies have significantly boosted the efficiency of the learning process by modifying loss functions, thereby amplifying the effectiveness of anomaly detection  \cite{zhao2020error, zhao2021synergistic}. These research efforts have played a crucial role in improving the capability to identify anomalies.

However, these methods have not fully recognized the fundamental differences between anomaly detection and other scenarios.
We first define \textbf{homophily} as the high probability of a node being connected to other nodes with the same label, whereas \textbf{heterophily} is the opposite \cite{zhu2020beyond}. Similarly, a homophily (heterophily) edge is the one connecting two nodes with the same (different) label(s).
As shown in Figure \ref{fig:fraud_homo}, in the context of GAD, the weighted homophily distribution, where every class has an equal contribution, exhibits a distinct bimodal characteristic. This suggests a significant disparity in the level of homophily among various nodes, a characteristic absent in both homophilic and heterophilic graphs. For simplicity, we refer to these graphs as \textbf{generic} graphs. We have observed that currently, there is no metric to describe this distribution discrepancy. Therefore, in order to quantify the severity of this phenomenon, we introduced a new metric, \textbf{Class Homophily Variance} (CHV). It is designed to describe the degree of difference in the homophily distribution. In the following text, we analyze and discover that the value of this metric is significantly larger than others in anomaly detection scenarios. 
And we think that is why vanilla GNN models and GNNs designed for solving heterophily tasks \cite{pei2020geom,chien2020adaptive,yan2022two} are struggling to identify anomalies due to their inability to effectively handle this unique variance in homophily distribution.

Currently, some methods try to exploit the distribution of homophily in GAD but problems remain. For example, H$^2$-FDetector \cite{shi2022h2} opts for different aggregation relations for homophilic and heterophilic edges to aid in anomaly detection. GHRN \cite{gao2023addressing} utilizes high-pass filters to measure the degree of one-hop label change in the central node and uses node predictions to selectively remove heterophilic edges. GDN \cite{gao2023alleviating} attempts to identify the anomaly pattern to reduce the impact of heterophilic neighbors. Conversely, SparseGAD \cite{gong2023beyond} employs multilayer perceptron (MLP) combined with feature similarity for pruning and connecting nodes. However, issues still persist. Firstly, these methods are based on irreversible operations such as pruning, selection, and connection, which we refer to as modifications. These irreversible methods can disrupt the structural information of the graph to some extent, leading to performance declines in certain scenarios. Secondly, the vast distribution differences in original relationships render methods based on modifications less effective. These challenges inspire us to ask the question: can a model autonomously generate its relationships, with the original relations serving only as an auxiliary?

In order to answer this question, we propose a novel model named \textbf{\underline{H}}omophily \textbf{\underline{ed}}ge \textbf{\underline{Ge}}neration Graph Neural Network, abbreviated as \cmtt{HedGe}. The core of this model lies in generating homophilic edges rather than modification, to handle the problem of excessive differences in homophily distribution. It can effectively utilize latent neighborhood relationships in the feature space. Firstly, \cmtt{HedGe} applies position encoding to each node, providing additional contextual information for subsequent attention mechanisms. The model then calculates attention relationships between pairs of nodes, obtaining attention coefficient. Next, \cmtt{HedGe} employs a novel Edge Specific Gumbel Softmax mechanism to sample and generate homophilic adjacency matrices. This process is differentiable, ensuring the optimization efficiency of the model. 
The newly generated graphs and the original graph are then fed into GNNs for message passing. Simultaneously, the attention weights are used for weighted aggregation to integrate information from all nodes and several relationships are merged. Finally, to further enhance the model's performance, we modify the loss function to suppress the generation of heterophilic edges by the attention mechanism, ensuring the model focuses on extracting homophilic features.

To further validate the effectiveness and universality of our model, we design a special Heterophily Attack method to increase the CHV of the graph, thereby increasing the graph classification difficulty, to demonstrate the robustness of the proposed model on other tasks. Moreover, we also conducted experiments on several anomaly detection datasets, comparing the performance of the \cmtt{HedGe} model against other baselines. The edge-generating capability of \cmtt{HedGe} also makes edgeless classification possible on GNN. 
In summary, our contributions are as follows:

\begin{itemize}


\item We quantitatively describe the homophily distribution differences in GAD and theoretically discuss the impact of it. We also develop a novel Heterophily Attack to simulate high CHV scenarios in generic datasets.

\item {We design an innovative model that generates homophilic edges from scratch through attention mechanism and sampling methods. This mitigates the homophily distribution differences and also utilizes potential nodes.}

\item {We conduct extensive experiments on multiple benchmark datasets to demonstrate the effectiveness of our method and achieve the best performance in scenarios of graph anomaly detection, simulation, and edgeless node classification.}

\end{itemize}

\section{Analysis and PRELIMINARIES}
In this section, we first elucidate the concept of Class Homophily Variance and measure it on various datasets, followed by a theoretical analysis. Finally, we formulate our problem.

\subsection{Class Homophily Variance}

We introduce a new metric, Class Homophily Variance (CHV), to describe the variance in homophily distribution of a graph. This metric quantifies the homophily differences in classes among different nodes in the network. Specifically, CHV calculates the variance of node homophily across different classes, reflecting the degree of consistency in label distribution among different nodes in the network.

We define $\mathcal{H}(v) = \frac{\left|\{ u \in \mathcal{N}(v) : \text{label}(u) = \text{label}(v) \} \right|}{|\mathcal{N}(v)|}$ as a node's homophily value, where the $\mathcal{N}(v)$ is the neighborhood of the node $v$ and $|\mathcal{N}(v)|$ is the cardinality of set $\mathcal{N}(v)$. Next, we calculate the average homophily for each class $C$, $\mathcal{\bar{H}}(C) = \frac{\sum_{v \in V_C} \mathcal{H}(v)}{|V_C|}$, where $V_C$ is the set of node belonging to the class $C$. The formal definition of Class Homophily Variance is as follows.
\begin{definition}[Class Homophily Variance]\label{thm:df1}
Given a graph $\mathcal{G}$ with $k$ classes, and defining $\mathcal{S} = \{C_1, C_2, ..., C_k\}$ as the set of classes. Let the average inter-class homophily be $\mu = \frac{\sum_{C \in \mathcal{S}} \bar{\mathcal{H}}(C)}{|\mathcal{S}|}$. Then, the Class Homophily Variance of graph $\mathcal{G}$ is defined as follows,
\begin{equation}
\text{Var}(\bar{\mathcal{H}})_\mathcal{G} = \frac{\sum_{C \in \mathcal{S}} \left(\bar{\mathcal{H}}(C) - \mu \right)^2}{|\mathcal{S}|}.
\end{equation}
\end{definition}
Meanwhile, in order to judge the homophily difference in a single class, we design the in-class homophily variance, $\text{Var}_C(\mathcal{H}) = \frac{\sum_{v \in V_C} \left(\mathcal{H}(v) - \bar{\mathcal{H}}(C)\right)^2}{|V_C|}$. Finally, we calculate the weighed homophily average, $\mu_w = \frac{\sum_{i} w_i v_i}{\sum_{i} w_i}$, $w_i = \frac{1}{p_i}$, where $p_i$ is the class ratio of the node $v_i$. This metric is used to analyze the homophily average of a graph when every class has an equal contribution.

\subsection{Data Analysis}
\label{sec:Data}
We use our metrics to analyze several GAD datasets and compare them with some generic graphs, with the results shown in Table \ref{tab:metic} (Please find more results in Appendix \ref{sec:dataset}).

In GAD datasets, it is clear that the CHV is significantly higher than others. 
This is because in the GAD dataset, the homophily values of normal nodes are very close to 1, while the variance of homophily values for anomalous nodes is close to 0, a phenomenon that has also been corroborated by previous works \cite{gao2023addressing, gao2023alleviating, gong2023beyond}. Our metric quantitatively describes this phenomenon.
At the same time, the in-class homophily variance across all scenarios is relatively low, suggesting that within each type of dataset, the homophily distribution among classes is relatively balanced. Meanwhile, the weighted average in the GAD graphs tends to be neutral (close to 0.5), whereas in the homophily and heterophily graphs, the weighted average shows a clear bias. For instance, the weighted average for the Photo dataset is close to 0.8293, indicating a strong homophily tendency; while the value for the Squirrel dataset is only 0.2190, showing a distinct heterophily. 

These metrics highlight the stark differences in class distribution of GAD datasets compared to other datasets. This discrepancy accounts for why models that perform well on homophilic or heterophilic graphs fail to achieve similar results in graph anomaly detection. Our model is designed to alleviate the problem.

\begin{table}[t]
\centering
\caption{Homophily analysis on different datasets.}
\label{tab:metic}
\newcolumntype{C}{>{\centering\arraybackslash}X} 
\newcolumntype{Y}[1]{>{\centering\arraybackslash}p{#1}}
{\small
\begin{tabularx}{0.45\textwidth}{@{}Y{1.5cm}|Y{1.5cm}|CCC@{}}
\toprule
Types & Dataset & $\text{Var}(\bar{\mathcal{H}})_\mathcal{G}$ &  $\overline{\text{Var}}_C(\mathcal{H})$ & $\mu_w$ \\
\midrule
\multirow{2}{*}{Anomaly} & Amazon & 0.1655 & 0.0082 & 0.5579 \\
& YelpChi & 0.1101 & 0.0130 & 0.5373 \\
\midrule 
Homophily & Photo & 0.0171 & 0.0433 & 0.8293 \\
Heterophily & Squirrel & 0.0018 & 0.0320 &0.2190 \\

\bottomrule
\end{tabularx}}
\end{table}

\subsection{Theoretical Analysis}
In this part, we discuss the impact of inter-class homophily differences and CHV on the aggregation effectiveness of GNNs. To simplify the analysis, we use the binary classification problem in graph anomaly detection as an example to study the impact of homophily differences on the classification performance of GCN \cite{kipf2016semi}, thereby observing the effects on GNN models that utilize the homophily principle \cite{mcpherson2001birds}.

To clarify the assumption, based on previous works \cite{ma2021homophily, luan2023graph}, we analyze the Contextual Stochastic Block Model (CSBM) \cite{deshpande2018contextual}, which is often used to theoretically analyze the behavior of GNNs. We propose a variant of CSBM, CSBM for Class Homophily (CSBM-C). The graphs generated by CSBM-C contain two disjoint class of nodes, the two classes are respectively referred to as $\mathcal{C}_0$ and $\mathcal{C}_1$. For each node $i$, its original embedding $\mathbf{x}_i \sim N(\boldsymbol{\mu},\mathbf{I})$, where $\boldsymbol{\mu} = \boldsymbol{\mu}_k \in \mathbb{R}^l$, $i \in \mathcal{C}_k, k \in \{0, 1\}$, $\boldsymbol{\mu}_0 \neq \boldsymbol{\mu}_1$, and $l$ is the dimension of the embedding. For the nodes in $\mathcal{C}_0$ and $\mathcal{C}_1$, their degrees are $d$. Simultaneously, their neighbors are independently sampled. For node $i$, its neighbors comprise $h \cdot d$ nodes with the same label and $(1 - h) \cdot d$ nodes with a different label, where $h \in \left[0,1\right],h = h_k, i \in \mathcal{C}_k, k \in \{0, 1\}$. We denote the graph $\mathcal{G}$ generated by CSBM-C as $\mathcal{G} \sim \text{CSBM-C}(\boldsymbol{\mu}_0, \boldsymbol{\mu}_1, d, h_0, h_1)$. Simultaneously, we represent a single GCN aggregation as $\mathbf{h}_i = \frac{1}{d} \sum_{j \in \mathcal{N}(i)} \mathbf{x}_j$, where $\mathbf{h}_i$ is the representation obtained after $\mathbf{x}_i$  is convolved through a GCN and $\mathbf{h}_i \in \mathbf{h}$. Because the feature matrix being multiplied can be absorbed by the subsequent linear classifier, we simplify it here.
We analyze the effects of variations in $h_0$ and $h_1$ on classification and arrive at the following conclusion:

\begin{theorem}\label{thm:theorem1}
For a graph $\mathcal{G} \sim \text{CSBM-C}(\boldsymbol{\mu}_0, \boldsymbol{\mu}_1, d, h_0, h_1)$, for any node $i$ in $\mathcal{G}$, the smaller the value of $|h_0+h_1-1|$, the greater the probability that $\mathbf{h}_i$ will be misclassified by $\mathbf{h}$'s optimal linear classifier.
\end{theorem}

The proof of the Theorem \ref{thm:theorem1}  can be found in Appendix \ref{sec:proof}. At the same time, CHV is directly proportional to $(h_0-h_1)^2$ (Proof can be found in Appendix \ref{sec:proofc}) under this assumption. In the GAD datasets, due to the high homophily value of normal nodes (close to 1), and low homophily value of anomalies (close to 0), in extreme cases, if $h_0=0$ and $h_1=1$, then the CHV reaches its maximum value, and the probability of misclassification is the highest. In generic graphs, where $h_0$ and $h_1$ are close to each other and both near to $0$ or $1$, the value of CHV is relatively lower, and the probability of misclassification is also smaller. 
Classification can also be challenging when the CHV is small, such as when $h_0 = h_1 = 0.5$. However, our previous data analysis has shown that this situation does not occur in both GAD or generic datasets. At the same time, CHV describes the characteristics of the GAD scenario more intuitively and can be easily extended to multi-class scenarios.

\begin{figure*}[t]
\includegraphics[width=1.0\textwidth]{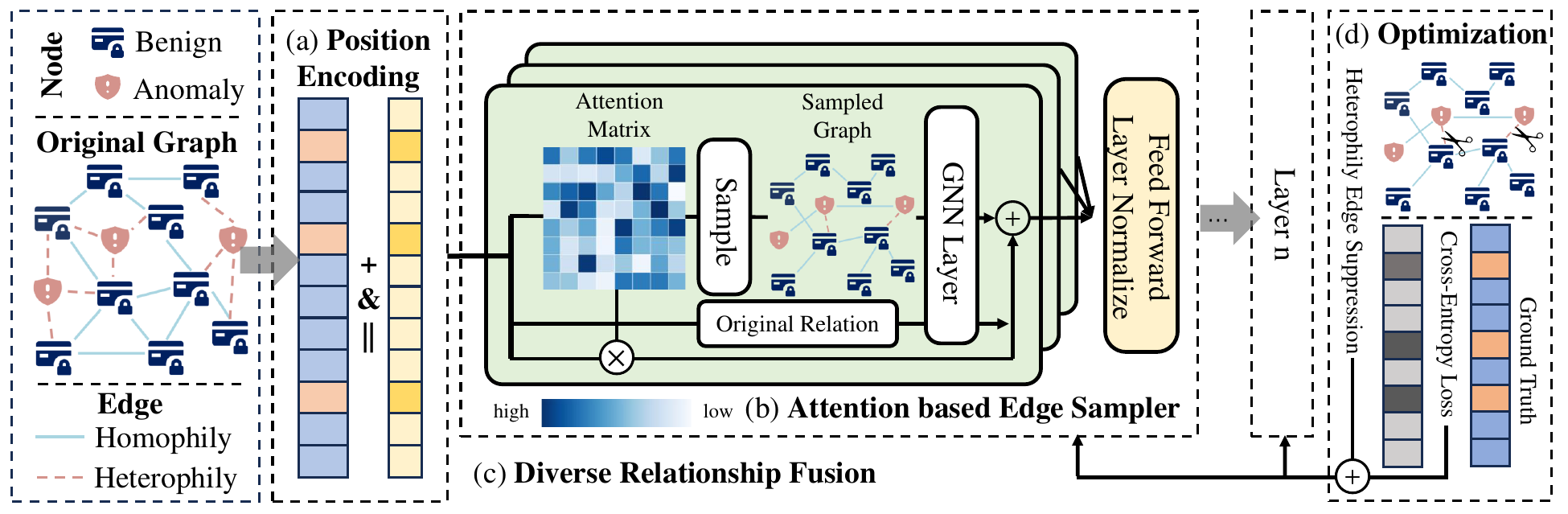}
\caption{ The overall architecture of the proposed \cmtt{HedGe}. (a) We first apply position encoding to enhance node information. (b) Then we calculate node relationships and sample new relationships through self-attention. (c) Next, we aggregate multiple relationships. (d) Finally, we penalize the attention matrix to suppress the generation of heterophilic edges. }
\label{fig:model}
\end{figure*}

\subsection{Problem Formulation}

In this task, we consider graph data as input, defining graph structured data as $\mathcal{G} = \{\mathcal{V}, \{\mathcal{A}_r\}, X\}$, where $\mathcal{V}$ represents the set of nodes, including both benign node and abnormal nodes, $\{\mathcal{A}_r\}$ represents the set of relations, with the total number of nodes $n$ in the dataset denoted as $|\mathcal{V}|$, and $r \in \{1, 2, ..., R\}$ indicating the relations, $R$ is the number of relations in the dataset. $\forall \mathcal{A} \in \{\mathcal{A}_r\}$, $\mathcal{A}$ is a binary matrix belonging to $\{0, 1\}^{n \times n}$. $\mathcal{A}$ represents an undirected graph, meaning if there is a connection between nodes $i$ and $j$, then $\mathcal{A}_{ij} = 1$ and $\mathcal{A}_{ji} = 1$. $X \in \mathbb{R}^{n \times f}$ represents the attributes of the nodes, and $f$ is the dimension. With the attribute of any node $v_i$, $X_{v_i} \in \mathbb{R}^f$. We define graph anomaly detection as a semi-supervised learning task. $\forall v \in \mathcal{V}, Y(v) \in \{0, 1\}$, where 0 represents a benign node and 1 represents an anomaly node. Meanwhile, $Y \in \{Y_{train}, Y_{val}, Y_{test}\}$, where $Y_{val}$ and $Y_{test}$ are not visible to the model during training. We train the model using $Y_{train}$ and use $Y_{val}$ to select the best model. We use GNN as the backbone model to achieve the lowest error on $Y_{test}$.

\section{proposed method}
In this section, we specifically introduce the technical details of \cmtt{HedGe} as shown in Figure \ref{fig:model}. We firstly describe how our model applies position encoding to each node to enhance the node's representation ability.  Next, we use an attention mechanism to generate attention matrices and then sample new relational graphs through the attention coefficients.  Following that, we combine multiple relationships, including original edges, generated relationships, and the attention matrix sum, and pass them to the next layer.  Finally, we present our optimization objective.

\subsection{Position Encoding}

\subsubsection{Degree Position Encoding}

Inspired by previous work \cite{wu2023decor}, we recognize that in a graph, the degree information of anomaly and benign nodes can effectively reflect their importance. This is especially evident when dealing with multi-relational graphs. The degree information is effective in capturing structural anomalies and identifying abnormal nodes. Abnormal nodes often differ noticeably from normal nodes in their degree distribution. This difference in distribution provides an intuitive and effective starting point for anomaly detection.

Assuming in a multi-relational graph, we have $R$ different types of relationships, and for each node $v$, we can define its degree under the $r$-th type of relationship as $\textbf{deg}_r(v)$. Therefore, the degree encoding of node $v$, $\textbf{PE}_{Degree}(v)$, can be constructed by 
\begin{equation}
\textbf{PE}_{Degree}(v) = \text{Concat}(\textbf{deg}_1(v), \textbf{deg}_2(v), \ldots, \textbf{deg}_R(v)),
\end{equation}
where we concatenate degrees under all $R$ types of relationships.

\subsubsection{Laplacian Position Encoding}
To capture the structural information of the entire graph, we utilize Laplacian Position Encoding to represent the position of nodes in the whole graph. The Laplacian matrix is an important tool for describing the spectral characteristics of a graph. It contains not only the connection information between nodes, but also reflects the structural information of the entire graph. Using the Laplacian matrix for position encoding enables the model to effectively capture and utilize global and local structural information when dealing with complex graph structures.

We calculate Laplacian Position Encoding by
\begin{equation} 
\begin{split} 
& \textbf{L} = \textbf{I} - \textbf{D}^{-\frac{1}{2}} \textbf{A} \textbf{D}^{-\frac{1}{2}},\\
& \textbf{L}\textbf{v} = \lambda \textbf{v},\\
& \textbf{V}_{\text{selected}} = \textbf{V}[:, 1:k+1],\\
& \textbf{pe}_i = \text{sign}_i \cdot \textbf{v}_i, \\
& \textbf{PE}_{L}^{r} = [ \textbf{pe}_1| \textbf{pe}_2| \ldots| \textbf{pe}_k], \\
& \textbf{PE}_{L}(v) = \text{Concat}(\textbf{PE}_{L}^{1}, \textbf{PE}_{L}^{2}, \ldots, \textbf{PE}_{L}^{R})[v,:],
\end{split}
\end{equation}
where $\textbf{A}$ is the adjacency matrix of the graph, $\textbf{D}$ is the diagonal degree matrix, $\textbf{I}$ is the identity matrix, $\lambda$ are the eigenvalues and $\textbf{v}$ are the corresponding eigenvectors of the Laplacian matrix respectively. Then we select the first $k$ non-trivial eigenvectors. Each of the selected eigenvectors is multiplied by a random $\text{sign}_i$ (+1 or -1) and $\textbf{pe}_i$ is the position encoding vector. 


\subsubsection{Label Encoding}

To mitigate the sample imbalance in anomaly node detection, we adopt downsampling strategy \cite{liu2008exploratory}, which doesn't include all training labels in a single training epoch. To further utilize label information, we use the labels from the training set that are not trained in a single epoch and set up label encoding. We classify unknown labels, including those in the validation set, test set, and the labels needed for the current training, as 2 for example, indicating they are unknown. Noting that labels 0 and 1 are preserved for to denote benign and anomaly nodes respectively.  Thus, $\textbf{L}(v)$ is the label mapping function, which outputs a label $l \in \{0,1,2\}$. We define an encoding function $\textbf{E}:\{0,1,2\} \rightarrow \mathbb{R}_{d}$, where $d$ is the embedding dimension, mapping the labels into a $d$-dimensional space. Label encoding  is calculated by
\begin{equation}
\textbf{PE}_{Label}(v) = \textbf{E}(\textbf{L}(v)).
\end{equation}
\subsubsection{Encoding Aggregation}
After calculating three types of position encodings, we perform position aggregation through 
\begin{equation}
\textbf{h}_v = \text{Concat}(\textbf{X}(v),\textbf{PE}_{Degree}(v) ,\textbf{PE}_{L}(v)) + \textbf{PE}_{Label}(v)
\end{equation}
to obtain the final embedding for each node. 

\subsection{Attention based Edge Sampler}
\subsubsection{Attention coefficient.}
After obtaining the embedding of the nodes, next we need to calculate the attention coefficients between each pair of nodes, in order to perform sampling and aggregation. We refer to the Transformer \cite{vaswani2017attention} architecture and adopt the scaled dot-product attention mechanism to calculate the attention coefficients,
\begin{equation}
a_{ij} = \text{Softmax}\left( \frac{(\textbf{W}_q \cdot \textbf{h}_i) \cdot (\textbf{W}_{k} \cdot \textbf{h}_j)}{\sqrt{d_k}} \right),
\end{equation}
where $a_{ij}$ represents the attention coefficient from node $i$ to $j$, $\textbf{W}_q$ and $\textbf{W}_k$ are learnable weight matrices, and $d_k$ denotes the dimension of the key vectors. The scaled dot-product attention mechanism, as compared to the attention mechanism introduced by GAT, offers a significant enhancement in computational efficiency. This is achieved by enabling the computation of attention across the entire graph in a matrix form using dot products. 

\subsubsection{Edge Specific Gumbel Softmax}
After obtaining the attention coefficients between each node pair, it's typical to employ methods like K-Nearest Neighbors (KNN) \cite{cover1967nearest} or Gumbel-Top-k trick \cite{kool2019stochastic} to identify or sample potential neighbors among the k-nearest ones. However, a limitation of these methods is that they result in each node having a predetermined, fixed number of connected neighbors, which does not accurately reflect the more variable and dynamic nature of real-world networks.  At the same time, these techniques involve selection and sorting operations, which are inherently non-differentiable. To address these problems, we choose to perform independent Gumbel-Softmax \cite{jang2016categorical} sampling on each attention coefficient.

This method is employed for sampling from a categorical distribution, facilitating gradient-based optimization. We focus on a binary case, employing a probability vector 
$\left[a,1 - a\right]$, where $a$ is in the range $\left[0, 1\right]$. Let \( G_i \) be independently sampled from a $Gumbel(0, 1)$ distribution, \( G_i = -\log(-\log(u_i)) \), where \( u_i \) is independently drawn from a uniform distribution in $[0, 1]$. The Edge Specific Gumbel Softmax distribution, ESGS for simple, is given by:
\begin{equation}
\text{ESGS}(a) = \left[\frac{\exp(L_1 / \tau)}{\exp(L_1 / \tau) + \exp(L_2 / \tau)},\frac{\exp(L_2 / \tau)}{\exp(L_1 / \tau) + \exp(L_2 / \tau)}\right],
\end{equation}
where $L_1 = \log(a) + G_1, L_2 = \log(1-a) +G_2$, and $\tau$ is the temperature parameter.
The temperature parameter controls the entropy of the output distribution, meaning that a lower temperature results in an output closer to a hard discrete distribution, while a higher temperature leads to a more uniform distribution. 
Next, we can sample whether there is an edge between nodes i and j by 
\begin{equation}
p_{ij} = \text{ESGS}(\lambda \cdot a_{ij}) \cdot \begin{bmatrix} 1 \\ 0 \end{bmatrix},
\end{equation}
where $\lambda$ is a hyperparameter to control the number of sampled edges. By employing a reparameterization trick \cite{kingma2014autoencoding}, 
we transfer the non-differentiable parts to the random sampling of $u_i$, ensuring the differentiability of the other parts.
Meanwhile, we use the straight-through trick to convert the continuous relaxation output of Gumbel-Softmax into one-hot encoding while ensuring an approximate gradient. Consequently, $ e_{ij} \approx p_{ij}, e_{ij} \in \{0, 1\}$. And we use $ e_{ij} $ to build a adjacency matrix $ \mathcal{A}_{ij}^{G} = \text{min}(e_{ij}+e_{ji}, 1)$ and send it to a GCN layer, for node $v_i$
\begin{equation}
\textbf{h}_{{G}_{v_i}} = \sigma \left( \sum_j \frac{1}{c_{ij}} \textbf{h}_{v_j} \textbf{W}_g \right),
\end{equation}
where $\sigma$ represents an activation function, $\textbf{W}_g$ is a learnable matrix and $c_{ij} = \sqrt{deg(i)deg(j)}$ in $\mathcal{A}^{G}$. In practice, we can choose to use the matrix form of the adjacency matrix for computation or adopt the discrete form of the adjacency matrix for acceleration, but this step will lose the gradient.

\subsection{Relation Fusion and Optimization}

\subsubsection{Diverse Relationship Fusion}
We apply GraphSAGE \cite{hamilton2017inductive} or GCN \cite{kipf2016semi} as the aggregation function to the original relationships, assuming we use GraphSAGE here. For the $r$-th type of relationship, the aggregation formula is 
\begin{equation}
\textbf{h}_{{O}_{v}}^r = \sigma(\textbf{W}_s \cdot \text{MEAN}(\{\textbf{h}_v\} \cup \{\textbf{h}_u, \forall u \in \mathcal{N}_r (v)\})) ,
\end{equation}
where $\mathcal{N}_r (v)$ is the neighborhood of node in relation $r$ and $\textbf{W}_s$ is a learnable matrix.

Then, we aggregate the original attention coefficients as weighted relationships and add up multiple relationships.
Now that we have three types of features, original relationship features, sampled relationship features, and attention coefficient features, and we merge these three relationships.
Finally, we use layer normalization $\text{LN}(.)$ and a feed-forward layer $\text{FFN}(.)$ to reduce numerical instability and facilitate learning. So the layer $k$ of the \cmtt{HedGe} is 
\begin{equation}
\begin{split}
& \textbf{h}_{A} = \sum_{j} a_{ij} \cdot \textbf{W}_v\cdot\textbf{h}_j, \\
& \textbf{h}^r_k = \text{Concat}(\textbf{h}_{G}^{r}, \textbf{h}_O^r ,\textbf{h}_{A}^r), \\
& \textbf{h}_{k} = \text{FFN}(\text{LN}(\text{MEAN}(\textbf{h}_{k-1}^1, \textbf{h}_{k-1}^2, \ldots, \textbf{h}_{k-1}^R))).
\end{split}
\end{equation}

\subsubsection{Optimization Objective}

In a \cmtt{HedGe} with k layers, the final representation of a node is denoted as $\textbf{h}_k$. For node classification, we employ a MLP and the Softmax function to get the possibility $p_v$, and optimize the model using the cross-entropy loss,
\begin{equation}
    \mathcal{L}_c = -\sum_{v \in \mathcal{V}} \left[ y \cdot \log(p_v) + (1 - y) \cdot \log(1 - p_v) \right].
\end{equation}

To suppress the generation of heterophilic edges, we add a heterophily edge suppression penalty term $\mathcal{L}_{p}$. This penalty term involves calculating the squared attention coefficients between nodes with differing labels, 
\begin{equation}
\mathcal{L}_{p} =\sum_{a} \left[ a_{ij}^2 \cdot 1_{\{ label(i) \neq label(j)\}}\right].
\end{equation}
Note that we only use labels from the training set for punishment to avoid label leakage. This penalty reduces the attention values for nodes with different labels, thereby decreasing the likelihood of generating heterophilic edges. By suppressing the generation of heterophily edges, we ensure that the generated edges have a lower CHV, making the dataset's relationships closer to generic graphs to enhance the learning capability of GNNs.

Our final optimization objective is
\begin{equation}
\mathcal{L} = \mathcal{L}_c + \alpha \mathcal{L}_{p} +\beta \|\theta\|^2_2,
\end{equation}
where $\alpha, \beta$ are hyperparameters and $\theta$ denotes the parameters of the model which need to be trained.

\begin{table*}[t]
\centering
\caption{Performance comparison of different models for anomaly detection.}
\label{tab:main-exper}
\newcolumntype{L}[1]{>{\raggedright\arraybackslash}p{#1}}
\newcolumntype{C}{>{\centering\arraybackslash}X} 

\begin{tabularx}{0.95\textwidth}{@{}L{2.3cm}|CC|CC|CC|CC|CC|CC@{}}
\toprule
\multirow{2}{*}{Method} & \multicolumn{2}{c|}{Amazon} & \multicolumn{2}{c|}{YelpChi} & \multicolumn{2}{c|}{BlogCatalog} & \multicolumn{2}{c|}{Reddit} & \multicolumn{2}{c|}{Amazon} & \multicolumn{2}{c}{YelpChi} \\
\cmidrule(r){2-13}
& AUC & AP & AUC & AP & AUC & AP & AUC & AP & AUC & AP & AUC & AP \\
\midrule
Training ratio & \multicolumn{8}{c|}{40\%} & \multicolumn{4}{c}{1\%} \\
\midrule
GCN                 & 85.32 & 35.14 & 60.34 & 23.69 & 88.09 & 46.52 & 65.04 & 6.70 & 77.97 & 23.74  & 54.64 & 17.68 \\
GAT                 & 93.04 & 60.67 & 59.82 & 23.48 & 71.25 & 21.45 & 65.19 & 5.53 & 80.56 & 33.29 & 54.87 & 16.95 \\
GraphSAGE           & 95.85 & 84.74 & 80.44 & 46.83 & 82.34 & 45.09 & 64.11 & 6.49 & 92.42 & 78.32  & 72.86 & 32.44 \\ 
\midrule
MixHop            & 96.03 & 86.44 & 79.56 & 45.29 & 88.31 & 48.20 & 65.63 & 5.44 & 92.42 & 78.32  & 72.86 & 32.44 \\
GPRGNN              & 94.75 & 76.75 & 73.11 & 32.97 & 81.93 & 43.46 & 62.93 & 5.45 & 93.52 & 74.51 & 67.66 & 28.70\\
\midrule
CAREGNN            & 88.48 & 69.24 & 77.96 & 36.63 & 69.40 & 27.13 & 67.21 & 6.72 & 87.27 & 73.93 & 75.29 & 36.07 \\
PCGNN              & 95.95 & 80.88 & 80.16 & 38.86 & 66.75 & 22.55 & 64.66 & 5.84 & 89.47 & 75.34 & 73.25 & 30.95 \\
AMNet               & 95.11 & 83.64 & 85.85 & 57.77 & 63.54 & 26.26 & 63.64 & 8.55 & 87.86 & 74.92 & 73.16 & 36.59 \\
H$^{2}$-FDetector   & 96.46 & 85.33 & 88.98 & 60.98 & 83.03 & 35.57 & 66.73 & 7.80 & 87.17 & 63.10 & 79.24 & 43.55 \\
BWGNN               & 97.99 & 90.09 & 90.22 & 63.78 & 79.37 & 37.39 & 71.37 & 8.96 & 89.10 & 80.40 & 77.52 & 37.66 \\
GDN                 & 97.10 & 87.37 & 90.26 & 66.42 & 70.70 & 28.55 & 69.19 & 6.94 & 83.30 & 61.48 & 73.39 & 38.68 \\
SparseGAD           & 97.03 & 89.17 & 88.61 & 66.01 & 70.16 & 25.60 & 66.12 & 5.98 & 93.67 & 81.40 & 78.73 & 40.77 \\

\midrule
\cmtt{HedGe}-w/pos        & 97.88 & 91.23 & 90.33 & 69.18 & 93.07 & 43.91 & 72.12 & 9.62 & 94.89 & 78.90 & 78.38 & 41.09 \\
\cmtt{HedGe}-w/sam        & 98.01 & 90.98 & 89.51 & 66.40 & 90.83 & 40.21 & 71.40 & 8.16 & 92.69 & 72.05 & 78.22 & 42.35 \\
\cmtt{HedGe}-w/loss       & 97.72 & 89.34  & 90.54 & 69.14 & 92.58 & 42.73 & 72.45 & 9.57 & 95.39 & 80.80 & 80.24 & 44.23 \\
\midrule 
\cmtt{HedGe}            & \textbf{98.25} & \textbf{92.30} & \textbf{91.29} & \textbf{70.68}  & \textbf{94.35} & \textbf{50.83} & \textbf{73.19} & \textbf{9.64} & \textbf{95.83} & \textbf{85.82} & \textbf{81.17} & \textbf{44.90}\\

\bottomrule
\end{tabularx}
\end{table*}

\section{Experiments}

In this section, we conduct a comprehensive analysis of the effectiveness of \cmtt{HedGe}. We evaluate it on four graph anomaly detection datasets, and then perform attacks on two generic datasets using our defined Heterophily Attack to simulate high Class Homophily Variance scenarios and test performance. We  explore its capabilities in edgeless node classification scenarios. We also execute ablation studies and engage in visualization techniques to verify that the model operates as anticipated.

\subsection{Experimental Setup}
\subsubsection{Datasets.} Following previous works \cite{zhu2020beyond,gao2023alleviating}, we conduct comprehensive evaluations of \cmtt{HedGe} in the GAD scenario on four benchmark datasets, including three real datasets: Amazon \cite{mcauley2013amateurs}, YelpChi \cite{rayana2015collective}, Reddit \cite{kumar2019predicting}, and an injected dataset, BlogCatalog \cite{tang2009relational}.  Additionally, we test two generic node classification datasets: Amazon co-purchase graphs Photo \cite{shchur2018pitfalls} and the citation graphs PubMed \cite{yang2016revisiting} and evaluate different models' performance under Heterophily Attack to simulate high CHV and show the universality of our model. For detailed descriptions and statistical data of these datasets, please refer to the Appendix \ref{sec:detail}.
\subsubsection{Baselines.}
We selected several representative models or the latest state-of-the-art models for comparison. For more detailed description, please refer to Appendix \ref{sec:base}. 

\begin{itemize}
    \item \textbf{GCN} \cite{kipf2016semi}, \textbf{GAT} \cite{velivckovic2018graph} and \textbf{GraphSAGE} \cite{hamilton2017inductive}. These models represent the most basic and widely used GNNs.
    \item \textbf{MixHop} \cite{abu2019mixhop} and \textbf{GPRGNN} \cite{chien2020adaptive} are models designed for overcoming over-smoothing and heterophily.
    \item \textbf{CAREGNN} \cite{dou2020enhancing}, \textbf{PCGNN} \cite{liu2021pick}, \textbf{H$^{2}$-FDetector} \cite{shi2022h2}, \textbf{AMNet} \cite{chai2022can}, \textbf{BWGNN} \cite{tang2022rethinking}, \textbf{GDN} \cite{gao2023alleviating} and \textbf{SparseGAD} \cite{zhu2020beyond} are anti-fraud models or GAD models. They are state-of-the-art models in GAD scenario.
    \item \textbf{KNN} \cite{cover1967nearest}, \textbf{SVM} \cite{boser1992training} and \textbf{MLP} \cite{rosenblatt1958perceptron} are classic machine learning algorithms. \textbf{Random Forest} \cite{breiman2001random}, \textbf{CATBoost} \cite{prokhorenkova2018catboost}, \textbf{XGBoost} \cite{chen2016xgboost} and \textbf{LightGBM} \cite{ke2017lightgbm} are classic decision tree-based machine learning algorithms. We compared them in edgeless classification scenarios.
\end{itemize}

\subsubsection{Metrics.}
Following previous works \cite{dou2020enhancing,gong2023beyond}, we use Area Under the Receiver Operating Characteristic curve (AUC) and Average Precision (AP) as the evaluation metrics in the anomaly detection scenario. AUC effectively measures the model's ability to discriminate between different classes by considering its performance across all possible classification thresholds. Notably, the sensitivity of AUC to imbalanced datasets makes it an ideal indicator for evaluating model performance in GAD scenarios. AP, which takes into account precision and recall at different thresholds, provides a more comprehensive performance evaluating for imbalanced classes. In generic datasets, we follow previous works \cite{kipf2016semi,chien2020adaptive}, we use accuracy as the evaluation criterion. 

\begin{table*}[th]
\centering
\caption{Performance comparison of models without edges for anomaly detection.}
\label{my-label}
\newcolumntype{L}[1]{>{\raggedright\arraybackslash}p{#1}}
\newcolumntype{C}{>{\centering\arraybackslash}X} 

\begin{tabularx}{0.95\textwidth}{@{}L{2.5cm}|CC|CC|CC|CC|CC|CC@{}}
\toprule
\multirow{2}{*}{Method} & \multicolumn{2}{c|}{Amazon} & \multicolumn{2}{c|}{YelpChi} & \multicolumn{2}{c|}{BlogCatalog} & \multicolumn{2}{c|}{Reddit} & \multicolumn{2}{c|}{Amazon} & \multicolumn{2}{c}{YelpChi} \\
\cmidrule(r){2-13}
& AUC & AP & AUC & AP & AUC & AP & AUC & AP & AUC & AP & AUC & AP \\
\midrule
Training ratio & \multicolumn{8}{c|}{40\%} & \multicolumn{4}{c}{1\%} \\
\midrule
KNN             & 91.87 & 81.60 & 75.85 & 36.69 & 62.71 & 20.81 & 57.96 & 4.80 & 86.36 & 68.05 & 64.68 & 22.16 \\
SVM             & 93.76 & 82.80 & 81.34 & 50.25 & 67.13 & 26.44 & 58.93 & 4.69 & 90.64 & 70.27 & 70.55 & 29.37 \\
MLP             & 97.04 & 86.95 & 82.42 & 31.27 & 58.15 & 13.30 & 59.64 & 4.71 & 74.00 & 42.15 & 70.42 & 31.27 \\
Random Forest   & 97.10 & 86.45 & 81.23 & 51.89 & 69.32 & 29.02 & 65.82 & 5.64 & 94.71 & 67.14 & 76.64 & 37.44 \\
CATBoost        & 97.20 & 89.44 & 83.11 & 54.36 & 66.70 & 27.96 & 62.24 & 4.58 & 95.08 & 82.65 & 73.19 & 32.16 \\
XGBoost         & 96.90 & 87.35 & 84.83 & 58.17 & 67.61 & 25.18 & 66.11 & 7.02 & 85.24 & 69.95 & 77.75 & 38.60 \\
LightGBM        & \textbf{97.95} & 89.30 & 85.71 & 60.08 & 72.36 & 29.54 & 66.78 & 6.51 & 93.41 & 67.14 & 76.25 & 35.68 \\
\midrule
\cmtt{HedGe}-w/edges            & 97.30 & \textbf{90.71} & \textbf{89.59} & \textbf{63.67} & \textbf{73.68} & \textbf{36.59} & \textbf{68.45} & \textbf{9.52} & \textbf{95.10} & \textbf{83.91} & \textbf{80.00} & \textbf{43.53}\\

\bottomrule
\end{tabularx}
\end{table*}

\subsection{Anomaly Detection Performance}

We conducted experiments on four benchmark datasets with 40\% of the labels used for training. To validate scenarios with scarce labels, we conducted experiments with only 1\% of the labels for training on Amazon and YelpChi dataset. We divided the remaining dataset into halves for the validation and test set respectively. 

Experimental results are presented in Table \ref{tab:main-exper}. It is evident that our model achieved the best performance across four anomaly detection datasets. We can see that the performance of vanilla GNNs is not optimal on GAD datasets, unable to adapt to high CHV situations, providing empirical evidence for Theorem \ref{thm:theorem1}.
Trained on 40\% of the Amazon, YelpChi, and BlogCatalog datasets, our model significantly improved performance. It increased the AP by at least 2\% and similarly raised AUC compared to competing models.
In scenarios with scarce labels, our improvement is also significant. For example, in the Amazon dataset with only 1\% of the labels used for training, our model achieved an absolute improvement of 2.16\% in AUC and 4.45\% in AP compared to the best-performing baseline, SparseGAD. It can also be observed that anomaly detection models did not perform satisfactorily on BlogCatalog, as their modifications and filtering of relationships severely hindered their detection of structural anomalies. However, our model retained the original structure and performed well on this dataset, exceeding the best-performing model, MixHop, by 6.04\% in AUC and outperforming the best GAD model, H$^2$-FDetector, by an impressive 11.32\% in AUC and 15.27\% in AP.

\begin{table}[t]
\centering
\caption{Class Homophily Variance under Heterophily Attack}
\label{tab:heter}
\newcolumntype{C}{>{\centering\arraybackslash}X} 
{\small
\begin{tabularx}{0.45\textwidth}{@{}CCCCCCC@{}}
\toprule
Ratio & 0\% & 1\% & 3\% & 5\% & 7\% & 10\% \\ 
\midrule
Photo & 0.0171 & 0.0194 & 0.0278 & 0.0401 & 0.0654 & 0.0754 \\
PubMed & 0.0044 & 0.0066 & 0.0137 & 0.0247 & 0.0410 & 0.0814 \\
\bottomrule
\end{tabularx}}
\end{table}

\begin{figure}[t]
    \centering
    \begin{minipage}{0.47\textwidth}
        \begin{subfigure}[b]{0.49\linewidth}
            \centering
            \includegraphics[width=\textwidth]{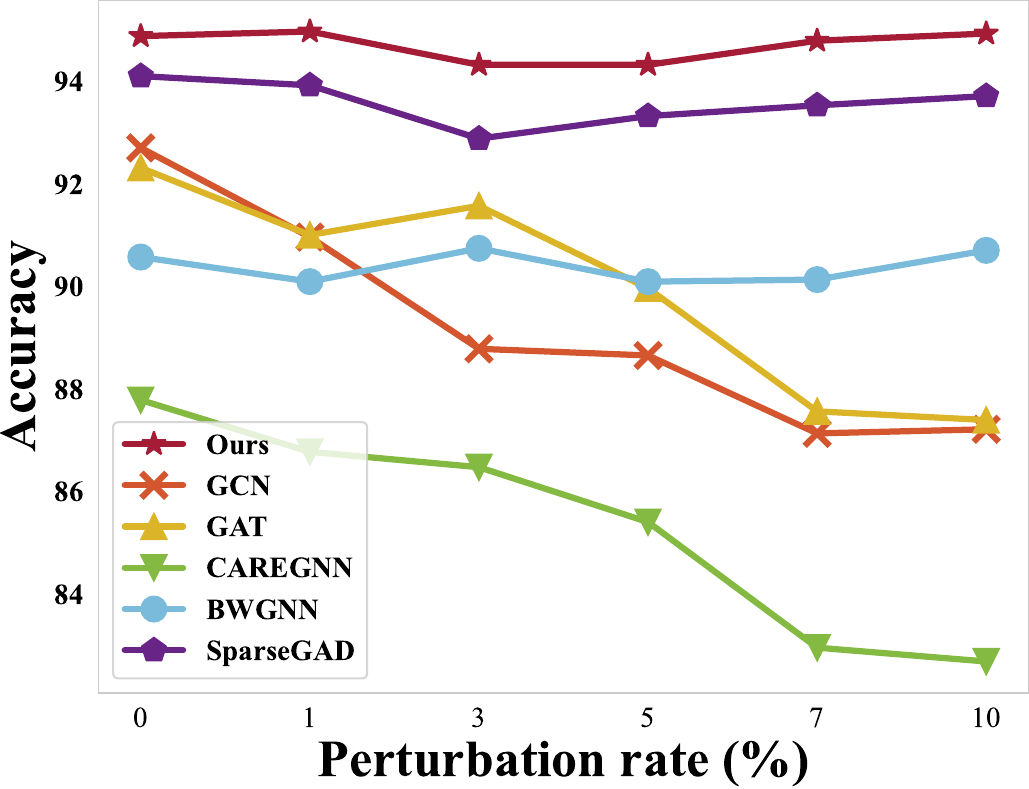}

            \caption{Photo}
        \end{subfigure}
        \begin{subfigure}[b]{0.49\linewidth}
            \centering
            \includegraphics[width=\textwidth]{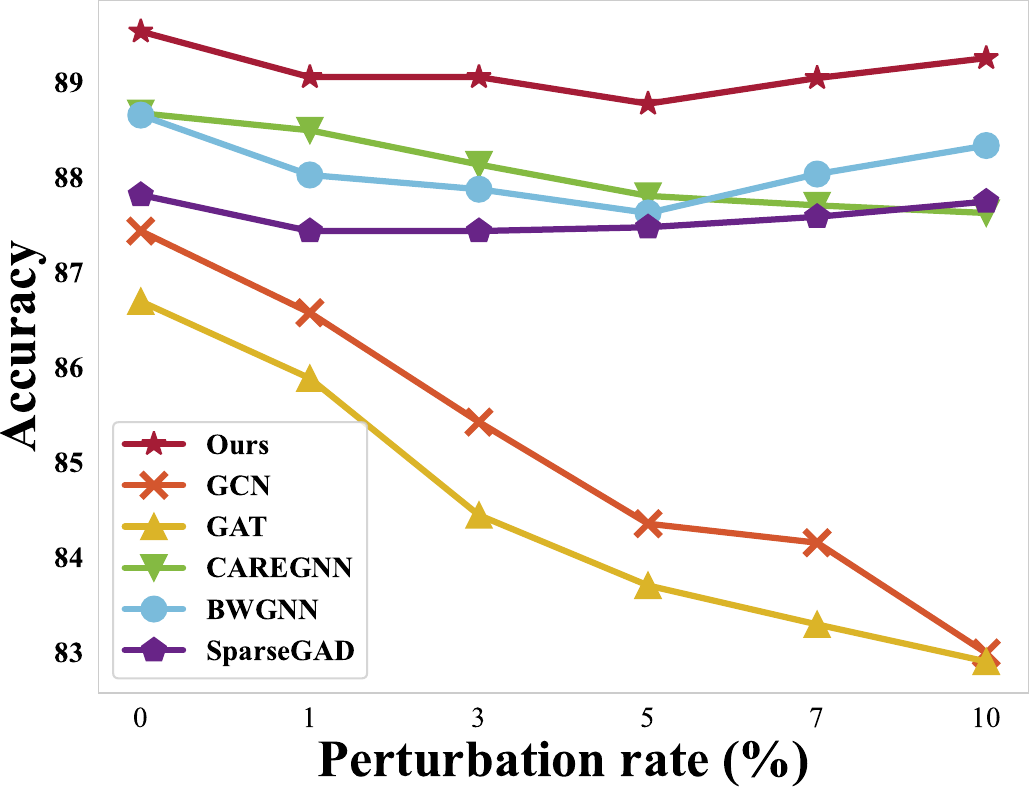}

            \caption{PubMed}
        \end{subfigure}

        \caption{Accuracy for different models under Heterophily Attack to increase Class Homophily Variance.}
        \label{fig:nonanomaly}
    \end{minipage}

\end{figure}

\subsection{Heterophily Attack and Generic Datasets}

To validate the adaptability of our model to high CHV in any scenario, we conducted tests on generic datasets and proposed Heterophily Attack to increase the CHV to simulate GAD datasets.


\subsubsection{Heterophily Attack}

Heterophily Attack is simple yet effective. To simulate anomaly detection datasets in generic datasets and increase the dataset's CHV, we targeted a single class for the attack. Suppose the targeted class is $C$. During the attack, we delete some edges $e$ where both end nodes $v$ belong to the attack class $C$, and then add some edges where one end node $v$ belongs to class $C$, and the other end node $v$ does not belong to class $C$. (Please refer to Appendix \ref{sec:heterattack} for more detailed description.) Our method has been tested and proven to effectively increase the CHV of datasets. Therefore, this method has simulated the most significant feature of GAD datasets on generic datasets. As shown in Table \ref{tab:heter}, on the PubMed dataset which contains only three classes, a 10\% edge perturbation increased the CHV by approximately 18 times. Although the Photo dataset has eight classes, a 10\% edge perturbation could still increase its CHV by more than four times.

\subsubsection{Results under Heterophily Attack}
We implemented Heter-ophily Attack on two datasets, Photo and PubMed, and compared several popular GNNs and GAD models. The experimental results are shown in Figure \ref{fig:nonanomaly}. 
We observed a significant decrease in the performance of vanilla GNNs like GCN and GAT when facing Heterophily Attack. This finding further confirms our previously proposed viewpoint, the significant difference between GAD and generic datasets is high CHV. It also proves to a certain extent the conclusion of Theorem \ref{thm:theorem1}, that the performance of vanilla models decreases when the CHV is very high.
Additionally, we noted that, apart from CAREGNN, the accuracy of other GAD models exhibited a slightly trend of initial decline followed by an increase, underscoring these models' adaptability to high CHV. Ultimately, our proposed \cmtt{HedGe} model outperformed all benchmark models in all attack ratios, proving its superior generality and robustness.

\subsection{Edgeless Node Classification}
Our \cmtt{HedGe} model, with its adaptive edge generation capability, 
has opened up new possibilities for GNN models in edgeless classification tasks. In our experiments, we removed the original relationships used as auxiliaries by \cmtt{HedGe} and relied solely on the edges generated by an attention-based edge sampler as input to the GNNs. By comparing with various classifiers that do not require edge information, we found that although tree-based classifiers like Random Forest and XGBoost have already shown excellent performance in GAD edgeless scenarios, even outperforming some GNNs and GAD models in many tasks, our \cmtt{HedGe} model still performed remarkably well in edgeless GAD tasks. 

As shown in Table 3, except for the AUC on Amazon, which did not achieve the best results, \cmtt{HedGe} led in all other evaluation metrics in other datasets and training ratios. In particular, on the Amazon and YelpChi, with 40\% labels for training, the AUC only dropped slightly (0.95\% and 1.7\%, respectively) compared to the situation with edges, effectively proving the efficacy of our edge generation strategy. We noticed that on the BlogCatalog dataset, there was a significant performance gap between edgeless and edged classification compared with other datasets, reflecting the importance of structural anomalies within this dataset, echoing the observation that GAD models focused on modification fail to achieve good results in previous experiments.


\subsection{Ablation Study}
In the detailed ablation study conducted on the \cmtt{HedGe} model, we focused on exploring the contribution and efficacy of each component of the model in the task of anomaly detection. The experiment was centered around the removal of three core components of the model: Position Encoding, Attention-based Edge Sampler, and the heterophily penalty term in the loss function. These variants were named \cmtt{HedGe}-w/pos, \cmtt{HedGe}-w/sam, and \cmtt{HedGe}-w/loss, respectively, with results shown in the last four rows of Table \ref{tab:main-exper}. The experiment results revealed several key findings. 

Firstly, with the exception of the 40\% training ratio for the Amazon dataset, the attention-based edge sampler significantly affected model performance, especially when the training set was extremely small. In sparse label environments, 1\% training ratio, the removal of the edge sampler led to a notable decrease in the AUC metric on the Amazon and YelpChi datasets, by 3.14\% and 2.95\%, respectively. This indicates that when training data is limited, the model's performance under high CHV is seriously constrained, and the edge sampler, by generating more homophilic edges, reduces the learning difficulty and thus improves GNN performance. Additionally, this series of ablation experiments also emphasized the importance of the homophily focus prior brought about by the heterophily edge suppression module in the loss function for improving model performance, as well as the critical role of position encoding in graph learning and graph anomaly detection.

\begin{figure}[t]
    \centering
    \begin{minipage}{0.47\textwidth}
        \begin{subfigure}[b]{0.49\linewidth}
            \centering
            \includegraphics[width=\linewidth]{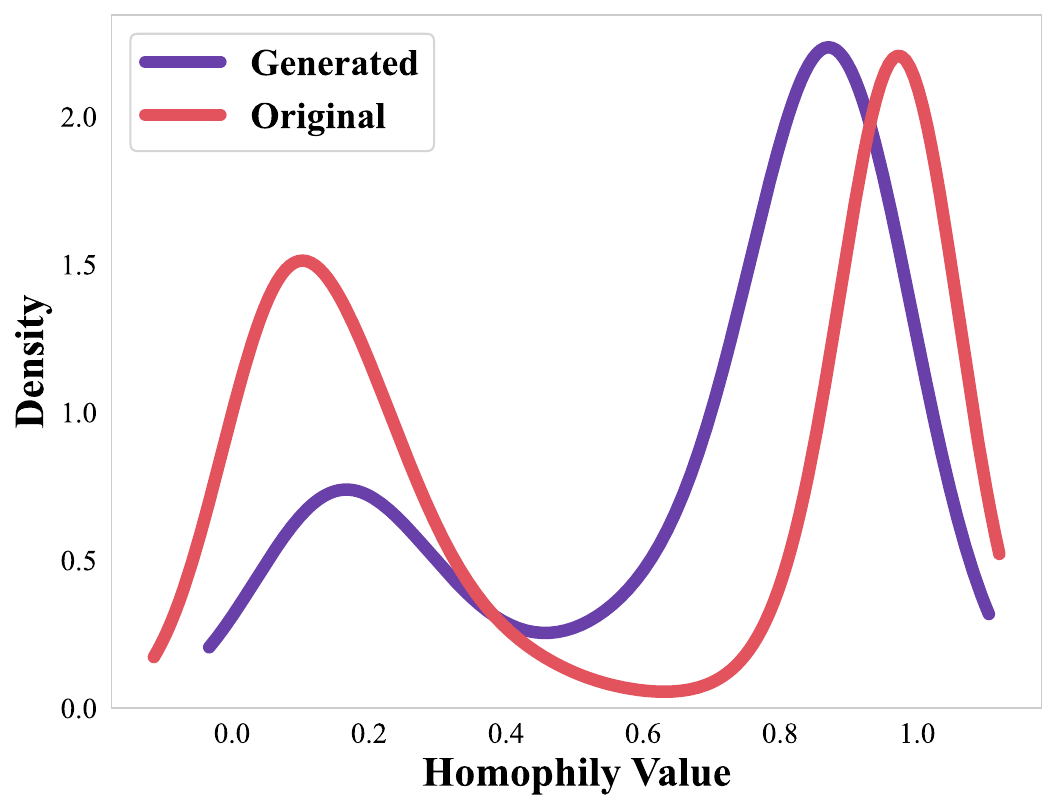}
            \caption{Amazon(0.1655/0.0312)}
        \end{subfigure}
        \hfill
        \begin{subfigure}[b]{0.49\linewidth}
            \centering
            \includegraphics[width=\linewidth]{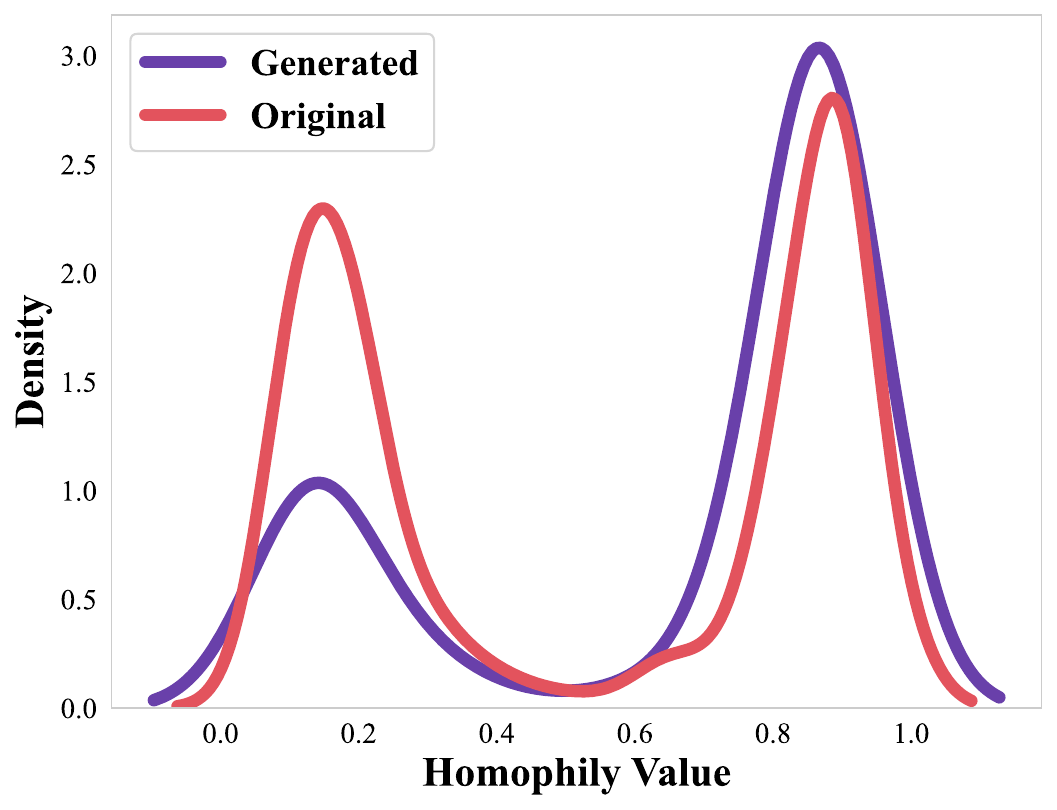}
            \caption{YelpChi(0.1101/0.0159)}
        \end{subfigure}
        \caption{Weighted homophily density distribution of the original and generated graphs. The pair of numbers enclosed by parentheses  presents the Class Homophily Variance of the original and generated graphs respectively.}
        \label{fig:casestudy}
    \end{minipage}

\end{figure}

\subsection{Interpretability}
Our model has successfully mitigated the original bimodal feature of the homophily distribution, as evidenced by the weighted homophily density distribution graphs as shown in Figure \ref{fig:casestudy}.  In the generated edges, the peak density on the right is more than three times higher than that on the left, whereas in the original graphs, two peaks have roughly the same density.  Furthermore, we also measure the CHV of the original and generated graphs. For Amazon, it decreased from 0.1655 to 0.0312, and for Yelp, it decreased from 0.1101 to 0.0159.  These reductions, both by at least a factor of five, demonstrate that the distribution of the edges generated by our method meets the expectations.

We also demonstrate the output results of different graph neural network models on the YelpChi after dimensionality reduction using t-SNE technique as shown in Figure \ref{fig:embeding}. Each subplot represents the output before the last layer of different models, where red represents anomalous nodes, and blue represents benign nodes. We have randomly downsampled the benign nodes to match the number of anomalous nodes for clarity.
Observing these visualizations, it is apparent that our model achieves more distinct separation in clustering compared to other models 
and produces clearer and more definitive groupings in t-SNE visualization. For instance, our model has significantly fewer overlaps in the red and blue areas compared to others, clearly showing an inverted U-shaped decision boundary. 


\begin{figure}[t]
    \centering
    \begin{subfigure}{0.12\textwidth}
        \centering
        \includegraphics[width=0.8\linewidth]{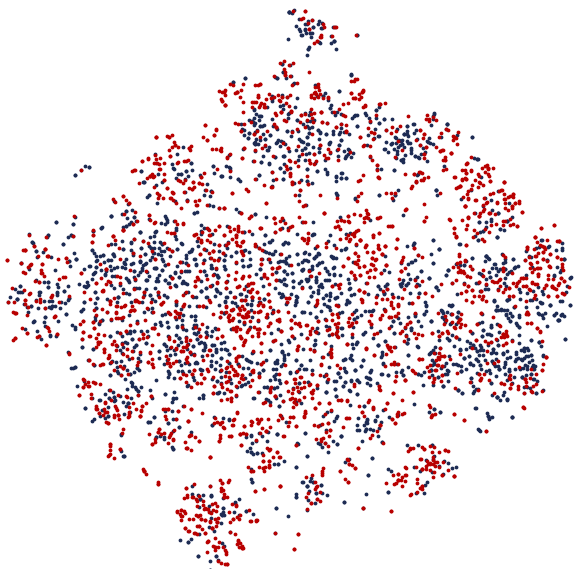}
        \vspace{5pt}
        \caption{Original}
        \vspace{10pt}
    \end{subfigure}%
    \begin{subfigure}{0.12\textwidth}
        \centering
        \includegraphics[width=0.8\linewidth]{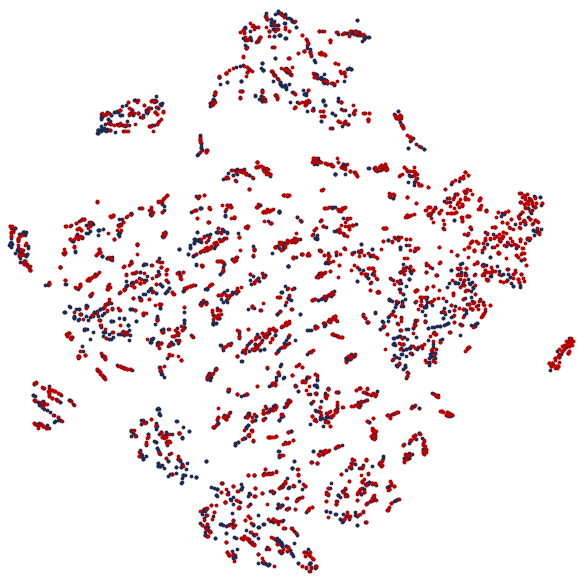}
        \vspace{5pt}
        \caption{GCN}
        \vspace{10pt}
    \end{subfigure}%
    \begin{subfigure}{0.12\textwidth}
        \centering
        \includegraphics[width=0.8\linewidth]{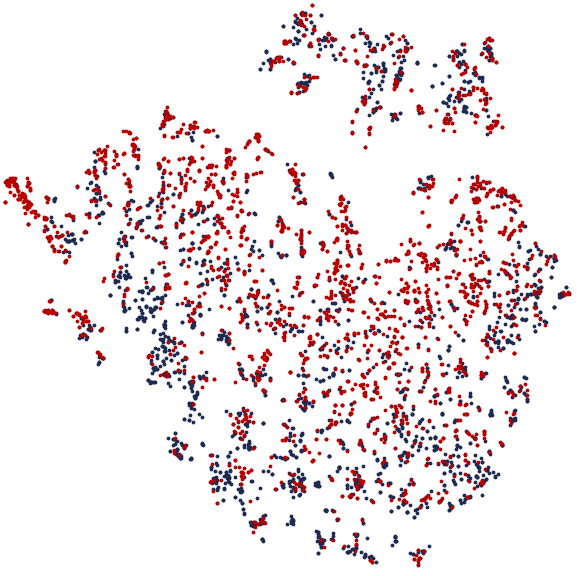}
        \vspace{5pt}
        \caption{GAT}
        \vspace{10pt}
    \end{subfigure}%
    \begin{subfigure}{0.12\textwidth}
        \centering
        \includegraphics[width=0.8\linewidth]{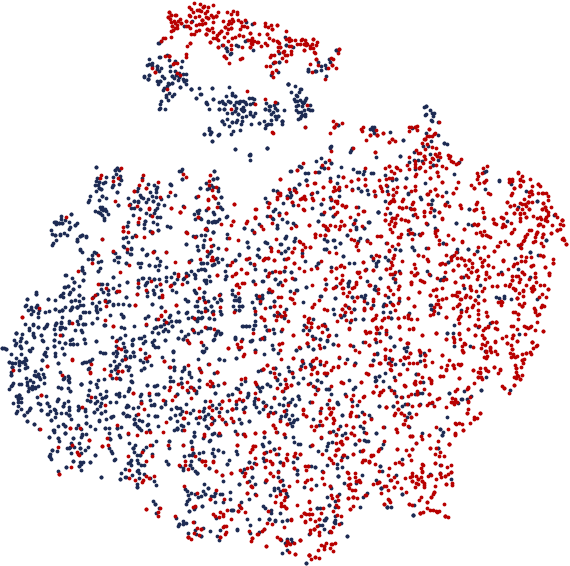}
        \vspace{5pt}
        \caption{GraphSAGE}
        \vspace{10pt}
    \end{subfigure}
    
    \begin{subfigure}{0.12\textwidth}
        \centering
        \includegraphics[width=0.8\linewidth]{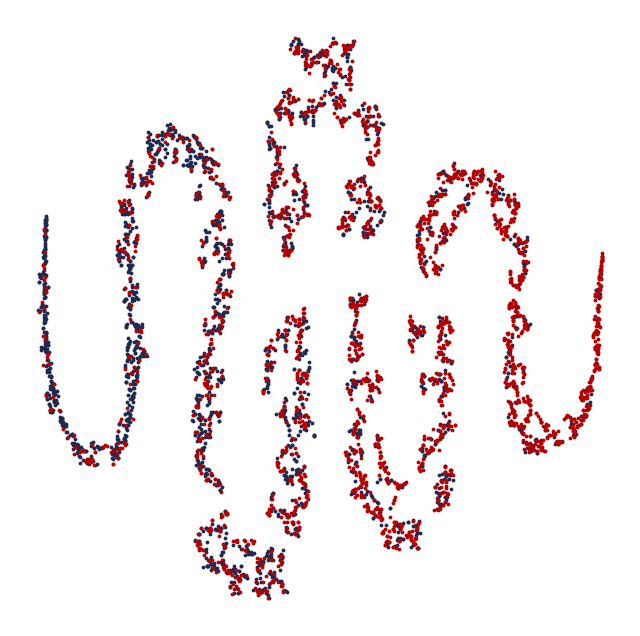}
        \vspace{5pt}
        \caption{GPRGNN}
        \vspace{10pt}
    \end{subfigure}%
    \begin{subfigure}{0.12\textwidth}
        \centering
        \includegraphics[width=0.8\linewidth]{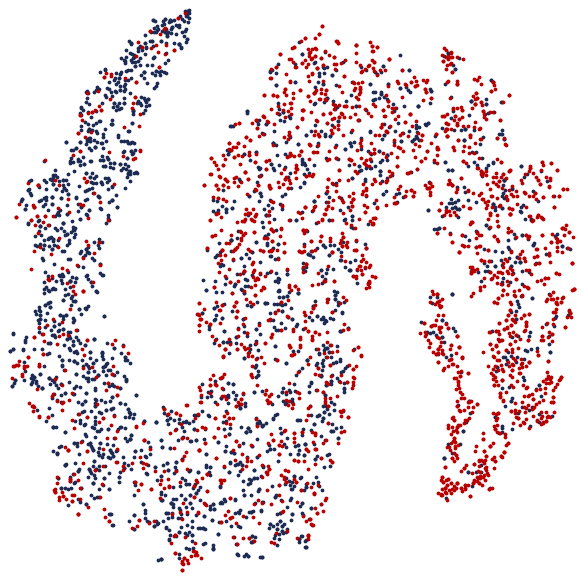}
        \vspace{5pt}
        \caption{CAREGNN}
        \vspace{10pt}
    \end{subfigure}%
        \begin{subfigure}{0.12\textwidth}
        \centering
        \includegraphics[width=0.8\linewidth]{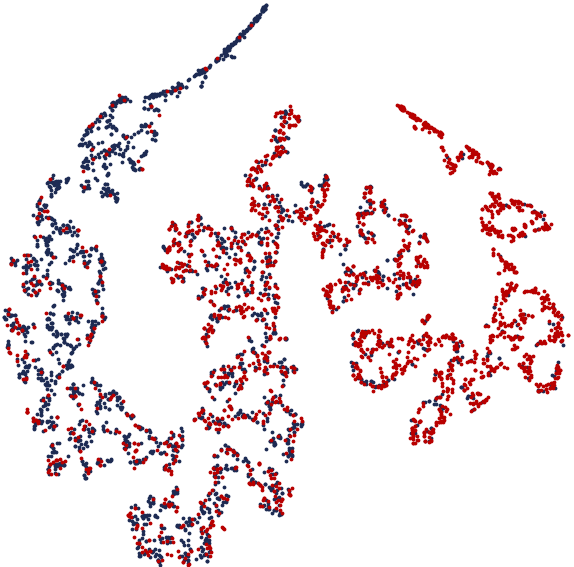}
        \vspace{5pt}
        \caption{PCGNN}
        \vspace{10pt}
    \end{subfigure}%
    \begin{subfigure}{0.12\textwidth}
        \centering
        \includegraphics[width=0.8\linewidth]{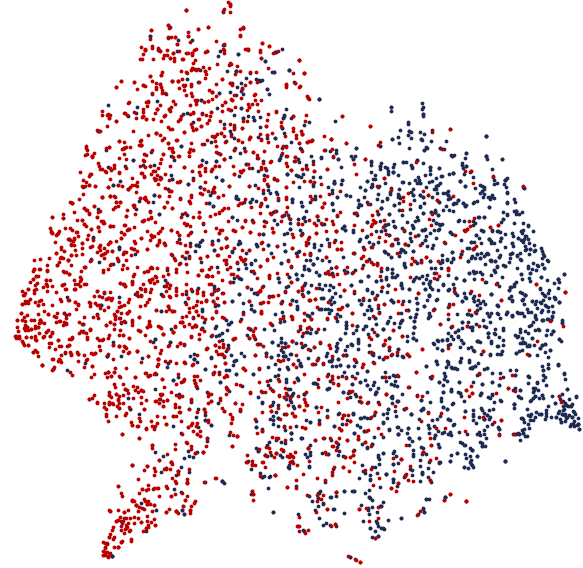}
        \vspace{5pt}
        \caption{AMNet}
        \vspace{10pt}
    \end{subfigure}

    \begin{subfigure}{0.12\textwidth}
        \centering
        \includegraphics[width=0.8\linewidth]{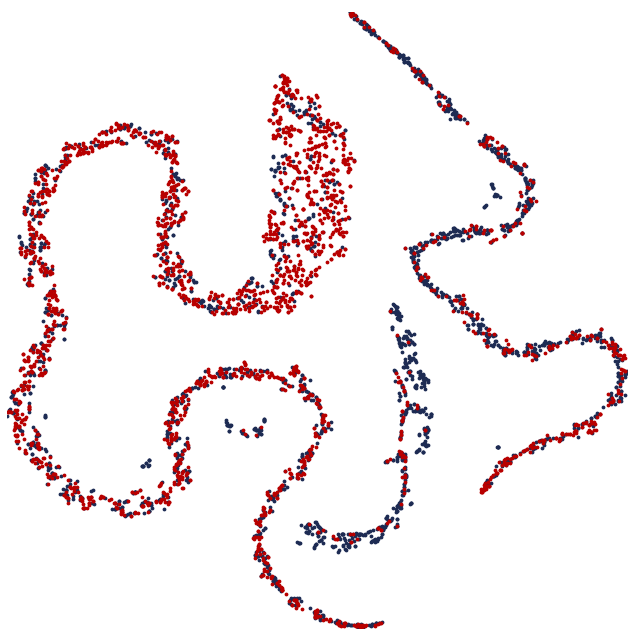}
        \vspace{5pt}
        \caption{H$^{2}$-FDetector}
    \end{subfigure}%
    \begin{subfigure}{0.12\textwidth}
        \centering
        \includegraphics[width=0.8\linewidth]{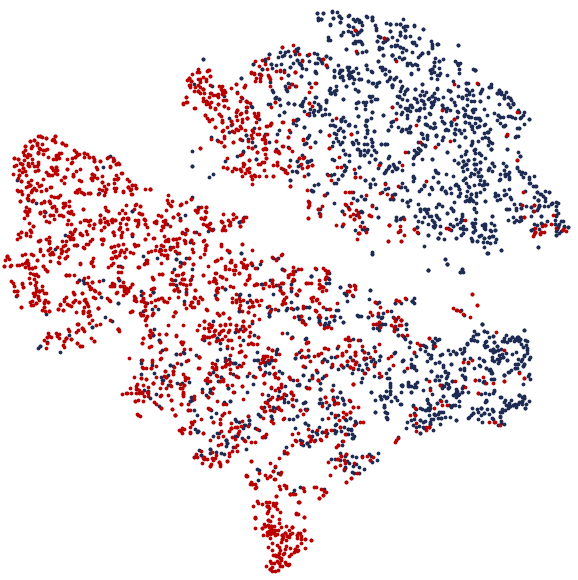}
        \vspace{5pt}
        \caption{BWGNN}
    \end{subfigure}%
    \begin{subfigure}{0.12\textwidth}
        \centering
        \includegraphics[width=0.8\linewidth]{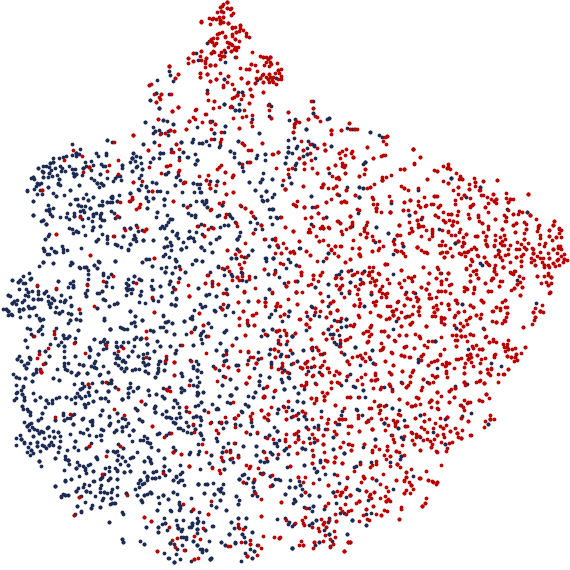}
        \vspace{5pt}
        \caption{GDN}
    \end{subfigure}%
    \begin{subfigure}{0.12\textwidth}
        \centering
        \includegraphics[width=0.8\linewidth]{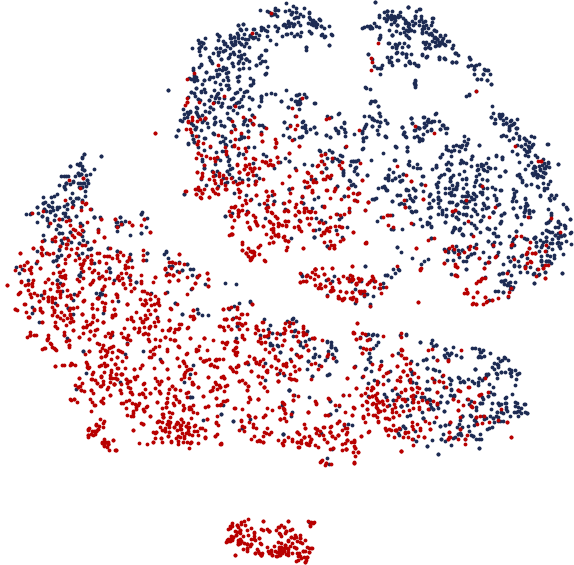}
        \vspace{5pt}
        \caption{Ours}
    \end{subfigure}
    \caption{t-SNE visualization of learned embeddings.}
    \label{fig:embeding}
\end{figure}

\section{Related Works}

\textbf{Graph Neural Network for Classification.}
Graph Neural Networks are deep learning models and are widely used in many fields such as drug repositioning research \cite{9_wang2022extending}, recommendation system \cite{liang2023ba} and relation classification \cite{4_li2019classifying,cheng2021efficient}. 
Vanilla GNNs like GCN \cite{kipf2016semi} are based on the homophily principle \cite{mcpherson2001birds}, but many real-world datasets are heterophilic and thus not suitable for them.
To alleviate this disparity, two strategies are employed \cite{zheng2022graph} : non-local neighbor and GNN architecture refinement. In models employing the non-local neighbor strategy, such as MixHop \cite{abu2019mixhop}, H2GCN \cite{zhu2020beyond}, UGCN \cite{jin2021universal}, and TDGNN \cite{wang2021tree}, the high-order neighbour mixing method is predominantly used, whereas models like Geom-GCN \cite{pei2020geom}, NL-GNN \cite{liu2021non}, and HOG-GNN \cite{wang2022powerful} primarily utilize the potential neighbor discovery method. Both approaches focus on extending local neighboring relationships to non-local ones. In models using the GNN Architecture Refinement strategy, such as FAGCN \cite{bo2021beyond} and WRGNN \cite{suresh2021breaking}, adaptive message aggregation is used; H2GCN \cite{zhu2020beyond} and WRGNN \cite{suresh2021breaking} employ the Ego-neighbor Separation method, and the GPRGNN \cite{chien2020adaptive} model utilizes the inter-layer combination method.
These methods have achieved good results on generic graphs. Due to the low CHV of these datasets, which is completely different from the GAD dataset, they do not perform well on GAD. Our models learns the representations between different classes and generate homophilic relations from scratch to input into GNNs, aligning with the data assumption paradigm of GNNs.

\hspace{-1em}\textbf{Graph-based Anomaly Detection.}
Many models for graph anomaly detection have been proposed. CAREGNN \cite{dou2020enhancing}, AOGNN \cite{huang2022auc} and PCGNN \cite{liu2021pick} employ reinforcement learning and resampling strategies to select neighborhoods. However, these models blindly aggregate neighboring nodes, leading to the camouflage of anomalies. AMNet \cite{chai2022can} and BWGNN \cite{tang2022rethinking} use multi-pass spectral filters to identify anomalies. Additionally, several works recognize the homophily differences between anomalies and benign nodes. GDN \cite{gao2023alleviating} uses a prototype vector to infer and update the distribution of anomaly features during training. SparseGAD \cite{gong2023beyond} sparsifies the structure of the target graph to effectively reduce noise and collaboratively learn node representations. GHRN \cite{gao2023addressing} trims inter-class edges by emphasizing and depicting the high-frequency components of graphs. 
There is still no way to quantitatively describe the difference in homophily between GAD and generic datasets. Most of these works are based on modifications. Because of the low homophily of anomalies, minor modifications  of relationships do not perform well. Also, modifications disturb structural information, leading to suboptimal performance in detecting structural anomalies. We propose CHV to describe homophily difference and, through generation rather than modification, combine original relationships as auxiliary for anomaly detection.

\section{Conclusion}
We provides a comprehensive analysis of how homophily distributions vary between anomaly detection datasets and others. It also proposes a novel metric, Class Homophily Variance, to effectively characterize these differences.
To address the issue of high CHV, we introduce \cmtt{HedGe}, which alleviates this problem by generating homophilic edges rather than modifying original relationships.
Experiments have demonstrated the effectiveness of our method in various scenarios including graph anomaly detection datasets, simulation and edgeless node classification, and have proven that the edges generated by the model have low CHV.

\bibliographystyle{ACM-Reference-Format}
\bibliography{sample-base}

\appendix
\section*{APPENDIX}
\setcounter{theorem}{0}
\setcounter{definition}{0}

\section{Theoretical result}
\subsection{Proof of Theorem \ref{thm:theorem1}}
\label{sec:proof}

\begin{theorem}
For a graph $\mathcal{G} \sim \text{CSBM-C}(\boldsymbol{\mu}_0, \boldsymbol{\mu}_1, d, h_0, h_1)$, for any node $i$ in $\mathcal{G}$, the smaller the value of $|h_0+h_1-1|$, the greater the probability that $\mathbf{h}_i$ will be misclassified by $\mathbf{h}$'s optimal linear classifier.
\end{theorem}

We referred to the approach of previous work \cite{ma2021homophily} and used the distance from the expected value to the optimal decision boundary to approximate the probability of misclassification. 
Unlike their work, which focused on proving when representations obtained by GNNs are better than the original representations, our focus is on the impact of class homophily differences on node classification.

\begin{proof}
Firstly, since the distribution of the node's neighborhood is known, and each neighbor can be treated as independent random variable, we can calculate the Gaussian distribution that $\mathbf{h}$ conforms to,
\begin{equation}
\mathbf{h}_i \sim 
\begin{cases} 
N\left(h_0\boldsymbol{\mu}_0+(1-h_0)\boldsymbol{\mu}_1,\frac{\mathbf{I}}{d} \right) & \text{if } i \in C_0 \\
N\left((1-h_1)\boldsymbol{\mu}_0+h_1\boldsymbol{\mu}_1,\frac{\mathbf{I}}{d} \right) & \text{if } i \in C_1 
\end{cases}.
\end{equation}
We can calculate the middle point and the direction of $\left( \mathbb{E}_{c_0}(\mathbf{h}), \mathbb{E}_{c_1}(\mathbf{h}) \right)$, which is $\mathbf{m} = \frac{(1+h_0-h_1)\boldsymbol{\mu}_0+(1-h_0+h_1)\boldsymbol{\mu}_1}{2}$, $\mathbf{w} = \frac{\boldsymbol{\mu}_0-\boldsymbol{\mu}_1}{\|\boldsymbol{\mu}_0-\boldsymbol{\mu}_1\|_2}$ respectively. Noting that $\mathbb{E}(.)$ means mathematical expectation.
From the analysis above, we can know that the optimal hyperplane that distinguishes the two features is orthogonal to $\mathbf{w}$ and passes through the point $\mathbf{m}$. We define this optimal hyperplane 
\begin{equation}
\mathcal{P} = \{ \mathbf{x} | \mathbf{w}^T\mathbf{x} - \mathbf{w}^T\mathbf{m} \}.
\end{equation}
In this context, we only articulate the scenario where node $i \in C_0$, as the case for $i \in C_1$ belonging to is symmetrical and identical. We define the probability of $\mathbf{h}_i$ being misclassified as 
\begin{equation}
    \mathbb{P}_{\text{mis}}(\mathbf{h}_i) =\mathbb{P}( \mathbf{w}^T\mathbf{h}_i - \mathbf{w}^T\mathbf{m} \leq 0 ), \text{for } i \in C_0.
\end{equation}
Because the variance of \( \mathbf{h}_i \) is independent of \( h_0 \) and \( h_1 \), under the same variance, the smaller the mathematical expectation of \( \mathbf{h}_i \) is from the decision boundary, the greater the probability of \( \mathbf{h}_i \) being misclassified. We calculate the distance of expected value of \( \mathbf{h}_i \) from the optimal decision boundary \( \mathcal{P} \) as
\begin{align}
    \text{dis}(\mathbf{h}_i) &= \frac{\left| \mathbf{w}^T\left(h_0\boldsymbol{\mu}_0+(1-h_0)\boldsymbol{\mu}_1 \right) - \mathbf{w}^T\frac{(1+h_0-h_1)\boldsymbol{\mu}_0+(1-h_0+h_1)\boldsymbol{\mu}_1}{2}  \right|}{\| \mathbf{w}^T \|_2} \\
    & = \left| \mathbf{w}^T\frac{(h_0+h_1-1)\boldsymbol{\mu}_0 + (1-h_0-h_1)\boldsymbol{\mu}_1 }{2} \right| \\
    & = \left| \mathbf{w}^T\frac{(h_0+h_1-1)(\boldsymbol{\mu}_0-\boldsymbol{\mu}_1)}{2} \right| \\
    & = \left| h_0+h_1-1 \right|\frac{\|\boldsymbol{\mu}_0-\boldsymbol{\mu}_1\|_2}{2} , \text{for } i \in C_0.
\end{align}

Therefore, the distance of $\mathbf{h}_i$ from the optimal decision boundary \( \mathcal{P} \) is directly proportional to $| h_0+h_1-1|$, thus completing the proof.
\end{proof}

\subsection{Class Homophily Variance under CSBM-C}
\label{sec:proofc}
\begin{theorem}
For a graph $\mathcal{G} \sim \text{CSBM-C}(\boldsymbol{\mu}_0, \boldsymbol{\mu}_1, d, h_0, h_1)$, the Class Homophily Variance of graph $\mathcal{G}$ is $\text{Var}(\bar{\mathcal{H}})_{\mathcal{G}} = \frac{( h_0 - h_1)^2}{4}$.
\end{theorem}
\begin{proof}
For node $i$, if $ i \in C_0$, then its homophily value $ \mathcal{H}(i) = \frac{h_0 \cdot d}{d} = h_0$, and if $ i \in C_1$, then its homophily value $ \mathcal{H}(i) = \frac{h_1 \cdot d}{d} = h_1$. 

Since each class has the same contribution, the average inter-class homophily is $ \mu = \frac{h_0+h_1}{2}$. Then we can calculate the  of graph $\mathcal{G}$,
\begin{align}
    \text{Var}(\bar{\mathcal{H}})_{\mathcal{G}} &= \frac{\left(h_0 -\frac{h_0+h_1}{2}\right)^2 + \left(h_1 -\frac{h_0+h_1}{2}\right)^2}{2} \\
    &= \frac{(h_0-h_1)^2}{4}.
\end{align}
This completes the proof.
\end{proof}

\section{Data Analysis on more datasets}
\label{sec:dataset}

\begin{table}[htbp]
\centering
\caption{Homophily analysis on different datasets.}
\label{tab:morehomo}
\newcolumntype{C}{>{\centering\arraybackslash}X} 
\newcolumntype{Y}[1]{>{\centering\arraybackslash}p{#1}}
{\small
\begin{tabularx}{0.45\textwidth}{@{}Y{1.5cm}|Y{1.5cm}|CCC@{}}
\toprule
Types & Dataset & $\text{Var}(\bar{\mathcal{H}})$ &  $\overline{\text{Var}}_C(\mathcal{H})$ & $\mu_w$ \\
\midrule
\multirow{4}{*}{Anomaly} & Amazon & 0.1655 & 0.0082 & 0.5579 \\
& YelpChi & 0.1101 & 0.0130 & 0.5373 \\
& BlogCatalog & 0.1378 & 0.0099 & 0.5737 \\
& Reddit & 0.1639 & 0.0129 & 0.5896 \\
\midrule 
\multirow{4}{*}{Homophily} 
& Cora & 0.0030 & 0.0854 & 0.8129 \\
& PubMed & 0.0044 & 0.1258 & 0.7766 \\
& Citeseer & 0.0200 & 0.1477 & 0.6861 \\
& Photo & 0.0171 & 0.0433 & 0.8293 \\

\midrule
\multirow{4}{*}{Heterophily}
 & Texas & 0.0198 & 0.0774 & 0.1080 \\
 & Chameleon & 0.0093 & 0.0472 & 0.2550 \\
 & Squirrel & 0.0018 & 0.0320 &0.2190 \\
 & Cornell & 0.0364 & 0.0771 & 0.1844 \\

\bottomrule
\end{tabularx}}
\end{table}

In order to more clearly illustrate the class homophily variance for different graph classification tasks, we provide statistical analysis results for 12 commonly used datasets, as shown in Table \ref{tab:morehomo}. These datasets include four GAD datasets, four homophilic graphs, and four heterophilic graphs. Notably, in these  datasets, none have a CHV exceeding 0.05, whereas in GAD datasets, none are below 0.1. Taking the  dataset Cornell as an example, it has the highest CHV at 0.0364, which is still about three times lower than the lowest in the GAD datasets, i.e., 0.1101 of YelpChi. This fact further confirms our view on the uniqueness of GAD datasets and is consistent with our previous discussion. The in-class homophily variance for all datasets is relatively low, but it is also evident that the in-class homophily variance of GAD datasets is smaller. Furthermore, the weighted average indicators further indicate that, compared to homophilic and heterophilic graphs, GAD datasets do not show a significant tendency in terms of homophily and are all quite close to 0.5. In contrast, homophilic or heterophilic graphs tend to be closer to either 0 or 1.

\section{Heterophily Attack}
\label{sec:heterattack}
In this section, we provide a detailed algorithmic description of our heterophily attack. 
We based on the given adjacency matrix and the attack ratio, calculate the total number of edges that need to be modified. 
Next, we identify the nodes that need to be attacked. If there is an edge between two nodes that are both marked with an attack label, then this edge may be removed. Afterward, we randomly add edges between nodes with attack labels and those without until the predetermined number of modifications is reached. To ensure the graph is undirected, we attack the edges above the diagonal, zero out everything below the diagonal, and then obtain an undirected graph by adding the adjacency matrix to its transpose.
Please refer to Algorithm \ref{alg:heter} for the pseudocode.

\begin{algorithm}
\caption{Heterophily Attack}
\label{alg:heter}
\begin{algorithmic}[1]
\Require    
The adjacency matrix of the graph $adjacency\_matrix$, 
the number of nodes $num\_nodes$, 
labels of all nodes $labels$, 
class need to be attacked $attack\_label$, 
attack ratio $ratio$ 

\Ensure Modified adjacency matrix

\For{$i = 1$ to $num\_nodes$}
    \For{$j = 1$ to $i$}
        \State $adjacency\_matrix[i, j] \gets 0$
    \EndFor
\EndFor

\State $num\_modifications \gets \text{sum}(adjacency\_matrix) \times ratio$

\State $attack\_index \gets$ indices of nodes where $labels = attack\_label$
\State Initialize $remove\_list$ as an empty list
\For{each $i$ in $attack\_index$}
    \For{each $j$ in $attack\_index$}
        \If{$adjacency\_matrix[i, j] \neq 0$}
            \State Append $(i, j)$ to $remove\_list$
        \EndIf
    \EndFor
\EndFor

\State $n \gets \min(num\_modifications, \text{len} (remove\_list))$
\State $remove\_index \gets$  random $n$ items from $remove\_list$
\For{each $index$ in $remove\_index$}
    \State $(i, j) \gets remove\_list[index]$
    \State $adjacency\_matrix[i, j] \gets 0$ 
    \State $adjacency\_matrix[j, i] \gets 0$ 
\EndFor

\State $i\_indices \gets$ indices of nodes where $labels = attack\_label$
\State $j\_indices \gets$ indices of nodes where $labels \neq attack\_label$

\For{$k$ in $1$ to $num\_modifications$}
    \Repeat
        \State $i \gets$ random item from $i\_indices$
        \State $j \gets$ random item from $j\_jindices$
        \If{$i < j$ and $adjacency\_matrix[i, j] = 0$}
            \State $adjacency\_matrix[i, j] \gets 1$
            \State \textbf{break}
        \EndIf
    \Until{a new edge is added}
\EndFor

\State $result\_matrix \gets adjacency\_matrix + adjacency\_matrix^\top$
\State \Return $result\_matrix$
\end{algorithmic}
\end{algorithm}

\begin{table*}[thbp]
\centering
\caption{Statistics of four anomaly detection datasets.}
\label{tab:ana}
\newcolumntype{C}{>{\centering\arraybackslash}X} 
\newcolumntype{Y}[1]{>{\centering\arraybackslash}p{#1}}
\begin{tabularx}{0.95\textwidth}{@{}Y{2cm}|CY{2.5cm}CCY{1.5cm}CCC@{}}
\toprule
\textbf{Dataset} & \textbf{Type}& \textbf{Scenarios} & \textbf{Node} & \textbf{Relations} & \textbf{Edge} &\textbf{Features} & \textbf{Anomalies} & \textbf{Rate} \\
\midrule
\multirow{3}{*}{YelpChi}&\multirow{3}{*}{Real} & \multirow{3}{*}{Review} & \multirow{3}{*}{45,954} & R-U-R & 49,315   &\multirow{3}{*}{32}& \multirow{3}{*}{6,674} & \multirow{3}{*}{14.52\%} \\
                       &  &                         &                        & R-S-R & 3,402,743 &&                    & \\
                       &  &                         &                        & R-T-R & 573,616   &&                    & \\
\midrule
\multirow{3}{*}{Amazon}&\multirow{3}{*}{Real} & \multirow{3}{*}{Review} & \multirow{3}{*}{11,944} & U-P-U & 175,608   &\multirow{3}{*}{25}& \multirow{3}{*}{821} & \multirow{3}{*}{9.50\%} \\
                       & &                         &                        & U-S-U & 3,566,479  &&                    & \\
                       & &                         &                         & U-V-U & 1,036,737 &&                    & \\
\midrule
Reddit      & Real &Social Networks & 10,984  & - & 175,608 & 64& 366 & 3.33\% \\
\midrule
BlogCatalog & Inject &Social Networks & 5,196   & - & 171,743 &8,189& 300 & 5.77\% \\
\bottomrule
\end{tabularx}
\end{table*}

\begin{table}[t]
\centering
\caption{Statistics of two generic datasets.}
\label{tab:gen}
\newcolumntype{C}{>{\centering\arraybackslash}X} 
\newcolumntype{Y}[1]{>{\centering\arraybackslash}p{#1}}
\begin{tabularx}{0.40\textwidth}{@{}C|CCCC@{}}
\toprule
\textbf{Dataset} & \textbf{Classes} & \textbf{Features} & \textbf{Nodes} & \textbf{Edge} \\
\midrule
Photo & 8 & 745 & 7,650 & 119,081 \\
PubMed & 3 & 500 & 19,717 & 44,324 \\
\bottomrule
\end{tabularx}
\end{table}

\section{EXPERIMENT settings}
\subsection{Weighted Homophily Density Distribution}
Here, we clarify how to draw the Weighted Homophily Density Distribution graph. We first calculate the homophily value $\mathcal{H}(v)$ for each node, as well as its weight $w$, where $w$ is the reciprocal of the proportion of its class in all nodes. Then, we use a kernel density estimator to fit the data and form a curve for easier visualization.

\subsection{Detailed Description of the Datasets}
\label{sec:detail}

Here, we provide a detailed description of each dataset.

The YelpChi dataset \cite{rayana2015collective} collects hotel and restaurant reviews from Yelp. This dataset treats reviews as nodes and establishes three types of relationships: 1) R-U-R: between reviews published by the same user, 2) R-S-R: between reviews of the same star level for the same product, and 3) R-T-R: between reviews posted in the same month for the same product. The Amazon dataset \cite{mcauley2013amateurs} focuses on product reviews in the musical instruments category on Amazon. Here, the nodes are users, and it also includes three types of relationships: 1) U-P-U: between users who have reviewed at least one common product, 2) U-S-U: between users who have given the same star rating within a week, and 3) U-V-U: between users who have the top 5\% mutual review text similarities. We build the edges of YelpChi and Amazon following previous work \cite{dou2020enhancing}. BlogCatalog \cite{tang2009relational} is a social blog directory whose main function is to allow users to discover and follow other blog authors. In this community, each member is considered as a node, and the interactions or follow relationships between members are seen as treated connecting these nodes. The attributes of these nodes are primarily used to describe various tags related to the users and their blog content. In this network, some nodes are deliberately set with structural and contextual anomalies according to previous work \cite{liu2021anomaly}. Reddit \cite{kumar2019predicting}, a well-known social media platform, has a forum post network that includes users banned by the platform, marked as anomalies. We load this dataset by PyGod \cite{liu2022pygod} package. The content of these users' posts is transformed into attribute vectors to represent their characteristics and behavioral patterns. The statistics of these four datasets can be found in Table \ref{tab:ana}. Note that in the Amazon dataset, there are 3305 nodes without labels.

PubMed \cite{yang2016revisiting} is a large biomedical literature database organized in a graph structure. Each piece of literature in the PubMed dataset can be considered a node, and the citation relationships between these pieces of literature form the edges. Photo \cite{shchur2018pitfalls} is derived from Amazon's co-purchase network of products. In this dataset, nodes typically represent products (in this context, products related to photography), and edges represent the co-purchasing relationships between these products. The statistics of these two generic datasets can be found in Table \ref{tab:gen}.

\begin{figure*}[ht]
    \centering
    \begin{minipage}{0.49\textwidth}
        \centering
        \includegraphics[width=\linewidth]{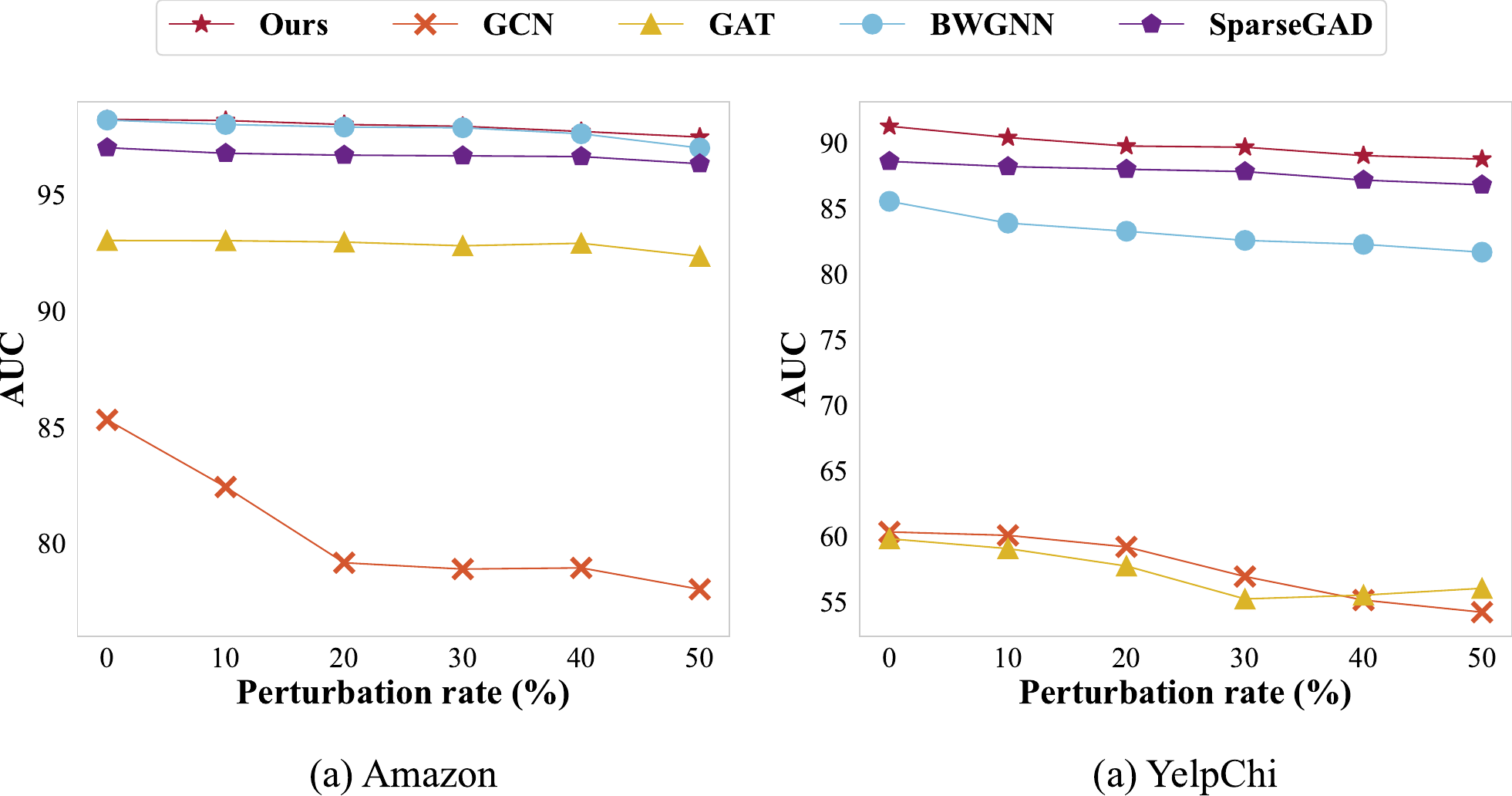}
        \caption{AUC under random attack.}
    \end{minipage}
    \hfill
    \begin{minipage}{0.49\textwidth}
        \centering
        \includegraphics[width=\linewidth]{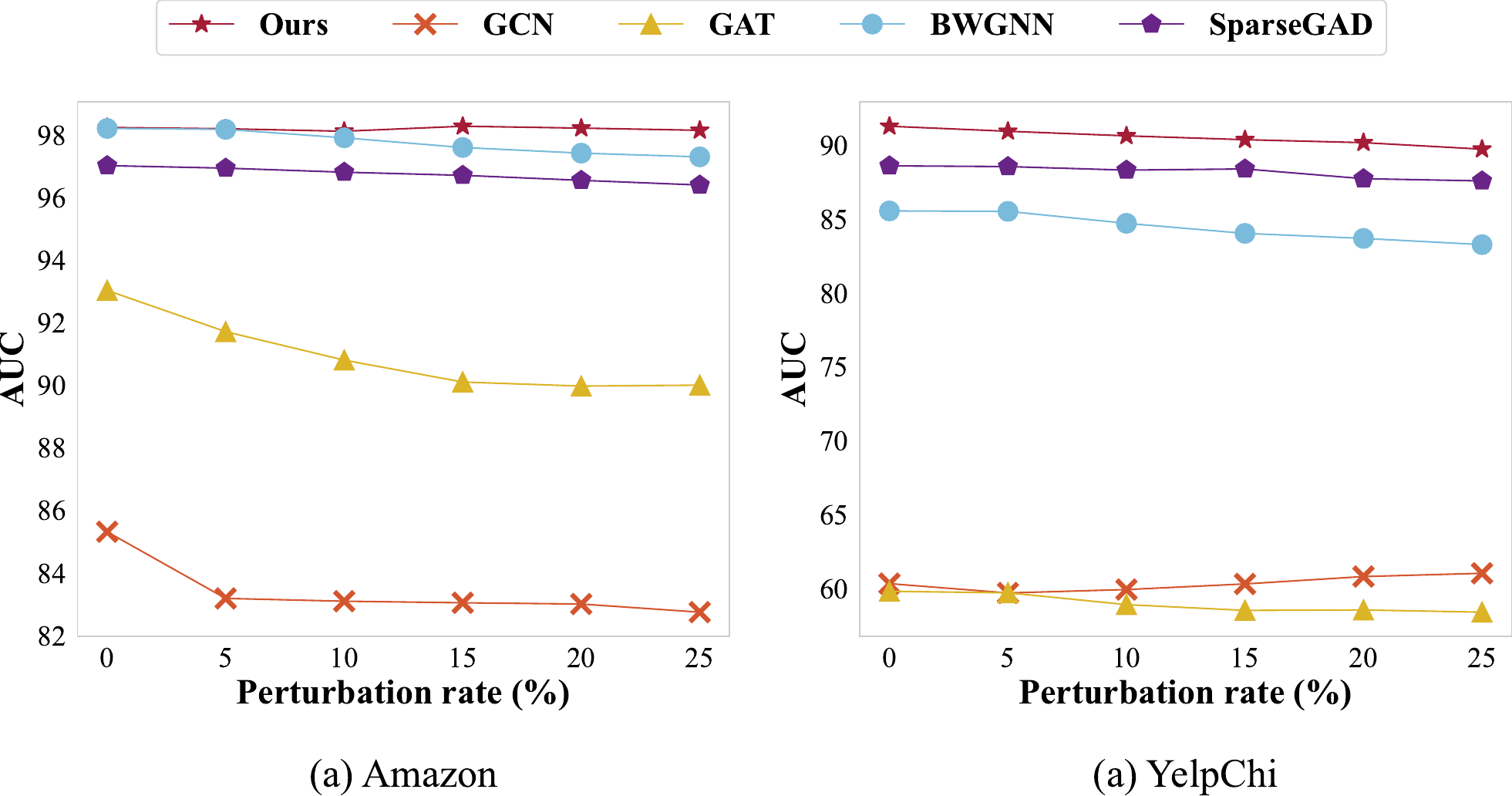}
        \caption{AUC under non-targeted attack.}
    \end{minipage}
\end{figure*}

\subsection{Baselines}
\label{sec:base}
Here, we provide a detailed description of our comparison methods.

The models below are common classic GNNs. They are widely used in various GNN tasks and possess excellent versatility.
\begin{itemize}
    \item \textbf{GCN} \cite{kipf2016semi} is a type of GNN that utilizes graph convolution operations to learn node representations in graph-structured data. It updates each node's features by aggregating the feature information of neighboring nodes.
   \item \textbf{GAT} \cite{velivckovic2018graph} introduces the attention mechanism into graph neural networks. In this model, nodes update their features by aggregating the features of their neighbors, weighted by dynamically computed attention scores. 
   \item \textbf{GraphSAGE} \cite{hamilton2017inductive} is an inductive learning graph neural network that updates the features of target nodes by sampling and aggregating a fixed-size set of neighboring nodes. 
\end{itemize}
The following models are optimized for the heterophily in graphs.
\begin{itemize}
    \item \textbf{MixHop} \cite{abu2019mixhop} utilizes a novel graph convolutional layer. It improves feature learning on graph data by blending information from neighbors at different hops.
    \item \textbf{GPRGNN} \cite{chien2020adaptive} combines the general PageRank algorithm with an adaptive mechanism. This is used for node importance assessment and attribute prediction on graph data.
\end{itemize}
Models specifically designed for anomaly detection usually exploit selection, pruning, and filtering techniques. We have chosen the state-of-the-art models along with some classic works.
\begin{itemize}
    \item \textbf{CAREGNN} \cite{dou2020enhancing} utilizes label-aware similarity to identify neighborhoods, employs reinforcement learning to determine the optimal number of neighbors, and aggregates selected neighbors across different relationships.
    \item \textbf{PCGNN} \cite{liu2021pick} effectively handles class imbalance by selectively sampling nodes and neighbors for improved learning and detection accuracy.
    \item \textbf{AMNet} \cite{chai2022can} adaptively combines multi-frequency signals for improved anomaly detection in graphs.
    \item \textbf{H$^2$-FDetector} \cite{shi2022h2} effectively identifies fraud by differentiating and aggregating information from both homophilic and heterophilic connections in a network.
    \item \textbf{BWGNN} \cite{tang2022rethinking} leverages spectral and spatial localized band-pass filters for enhanced anomaly detection in graphs, effectively addressing the right-shift spectral phenomenon.
    \item \textbf{GDN} \cite{gao2023alleviating} effectively addresses anomaly detection in graphs by dynamically adjusting to structural distribution shifts, optimizing for both anomalies and normal nodes.
    \item \textbf{SparseGAD} \cite{gong2023beyond} enhances detection quality by sparsifying graph structures and learning node representations to uncover hidden dependencies in relational data.
\end{itemize}

\subsection{Implementation Details}
\subsubsection{Experiment Platform.} We conducted experiments on a Linux server equipped with an Intel Xeon 5220 CPU, a Tesla V100 32GB GPU, and 64GB of RAM. In addition, for experiments conducted on YelpChi, we utilized an A800 80GB GPU to accelerate computations.

\subsubsection{Experiment Settings}
In anomaly detection datasets, when 40\% of the labels are used for training, we use 30\% as the validation set and 30\% as the test set. When 1\% of the labels are used for training, we use 49.5\% as the validation set and 49.5\% as the test set. When testing on the Photo and PubMed datasets, we use 40\% as the training set, 30\% as the validation set, and 30\% as the test set. We test on the validation set every 10 epochs, and select the model that performs best on the validation set to evaluate on the test set.
Inspired by CAREGNN \cite{dou2020enhancing}, small-batch data learning is beneficial in reducing model overfitting and improving model efficiency. Compared to random partitioning, we adopted the approach of ClusterGCN \cite{chiang2019cluster} to divide subgraphs on the Amazon, YelpChi, and PubMed datasets. This was done to accelerate model learning and reduce the occurrence of overfitting. Our experimental results are the average of 10 runs. For all baselines, if the original hyperparameters are provided, we use them. If not, we perform grid search using learning rates of in the set of \{0.01, 0.003, 0.001\} and the number of hidden layers of in \{16, 32, 64\}.

\section{More Experiments}
To further demonstrate the robustness of our model, we also show the precision changes of our model under two common attack methods on GAD datasets. First, let's introduce the two attack methods we adopted.

\begin{itemize}
    \item \textbf{Random Attack}: This type of attack randomly deletes a certain proportion of edges and then randomly adds a certain proportion of edges to create edge perturbations.
    \item \textbf{Non-targeted Attack}: This attack aims to target the entire graph rather than reducing the accuracy of certain nodes. Here, we chose the classic DICE attack \cite{zugner2019adversarial}.
\end{itemize}
In this experiment, we also randomly split the datasets into training, validation, and test sets with a ratio of 4:3:3. We used homogeneity graphs in both attacks, meaning the entire graph has only one type of relationship. Therefore, the performance of some models is inconsistent with that in heterogeneity graphs. Notably, BWGNN showed a significant decline on YelpChi.

It can be seen that vanilla GNNs, such as GAT and GCN, are more susceptible to attacks, showing a significant decline in performance from their original levels. In contrast, the GAD models demonstrates stronger robustness, with a smaller decline when faced with various attacks compared to vanilla GNNs. It is evident that our model exhibits strong robustness when facing multiple attacks, outperforming all baseline models comprehensively, and also has a relatively small decline in performance. These results fully demonstrate the effectiveness of our edge generation strategy in reducing the impact of attacks, thereby ensuring the model's robust performance in a changing environment.

\end{document}